\documentclass[hidelinks,journal,comsoc]{IEEEtran}
\usepackage{amsmath,amsfonts,amssymb}
\usepackage{amsthm}
\usepackage{array}
\usepackage[caption=false,font=normalsize,labelfont=sf,textfont=sf]{subfig}
\usepackage{textcomp}
\usepackage{stfloats}
\usepackage{url}
\usepackage{verbatim}
\usepackage{graphicx}
\usepackage{cite}
\usepackage{booktabs}
\usepackage{multirow}
\usepackage{bm}

\newtheorem{theorem}{Theorem}
\newtheorem{lemma}{Lemma}
\newtheorem{proposition}{Proposition}
\newtheorem{corollary}{Corollary}
\newtheorem{definition}{Definition}

\begin{document}

\title{Cognitive Graph Intelligence for Adaptive and Robust DDoS Attack Detection in Next Generation Networks}

\author{Mohammad Arif Hossain, ~\IEEEmembership{Member,~IEEE,} 
Yeahia Sarker, ~\IEEEmembership{Member,~IEEE}, Md Jafrin Hossain, Most. Humayra Khanom Rime, and
Nirwan Ansari,~\IEEEmembership{Life Fellow,~IEEE,}
        
\thanks{M. A. Hossain, Y. Sarker, and M. H. K. Rime are with Middle Tennessee State University, USA, M. J. Hossain is with C2C Tech, Germany, and N. Ansari is with New Jersey Institute of Technology, USA. Corresponding author: Mohammad Arif Hossain (e-mail: mohammad.hossain@mtsu.edu).}}

\markboth{IEEE Transactions on Cognitive Communications and Networking}%
{Cognitive Graph Intelligence for Adaptive and Robust DDoS Attack Detection}

\maketitle

\begin{abstract}
Distributed Denial-of-Service (DDoS) attacks threaten network availability, requiring a cognitive detection process that senses traffic, infers intent, and supports an adaptive response under severe class imbalance and non-stationary conditions. This paper proposes a Graph-based Generative Adversarial Network (GraphGAN) that serves as the cognitive detection engine for this task. GraphGAN captures the relational structure among traffic flows while addressing imbalance through adversarial generation of synthetic samples. Sequential flows are converted into $k$-nearest neighbor graphs using sliding windows to preserve feature-similarity and temporal dependencies among flows. The generator learns the distribution of DDoS attacks to synthesize realistic minority samples, while a Graph Convolutional Network (GCN)-based discriminator distinguishes real from synthetic graph data. A separate GCN classifier, trained on the balanced dataset, performs the final detection decision. Evaluations on four benchmark datasets show that GraphGAN achieves superior accuracy, precision, and recall compared to state-of-the-art approaches, particularly in data-scarce scenarios. By integrating temporal graph construction, adversarial augmentation, and GCN classification, GraphGAN effectively models coordinated attack behaviors and mitigates class imbalance, providing a robust and topology-aware solution for intrusion detection in data-constrained environments.
\end{abstract}


\section{Introduction}\label{intro}
\IEEEPARstart{D}{istributed} Denial-of-Service (DDoS) attacks have emerged as one of the most persistent and devastating threats in modern cybersecurity, causing billions of dollars in annual damages and disrupting critical infrastructure worldwide \cite{Wang_DDoS}. Unlike targeted intrusions, DDoS attacks exploit the distributed nature of compromised hosts to overwhelm services, making real-time mitigation extremely challenging. The scale and sophistication of these attacks continue to escalate: botnets such as Mirai have demonstrated the ability to mobilize millions of compromised IoT devices into coordinated flooding campaigns \cite{antonakakis2017understanding, liu_DT}, while modern multi-vector attacks combine volumetric, protocol, and application-layer techniques to evade single-point defenses \cite{zargar2013survey}. These developments highlight the inadequacy of static, rule-based protection mechanisms and motivate the need for intelligent, adaptive detection frameworks.

Traditional network intrusion detection systems (NIDS) rely primarily on signature-based rules and statistical anomaly detection, both of which struggle to adapt to rapidly evolving attack vectors \cite{jaafar2019review}. Signature-based systems require prior knowledge of attack patterns and are inherently reactive, limiting their ability to detect zero-day or polymorphic variants. Statistical anomaly detectors, while more general, often suffer from high false-positive rates and limited capability to model complex, non-linear relationships in high-dimensional traffic data \cite{hossain2023hybrid, hossain2023decentralized}. Meanwhile, the exponential growth in traffic volume and complexity has rendered conventional flow-based analysis insufficient, as processing flows independently fails to capture the interconnected and coordinated nature of modern network communications \cite{Luo_IDS, Liu_RL, Hossain_SL}.

With the proliferation of cloud computing, IoT ecosystems, and high-speed 5G/6G networks, network defense increasingly relies on cognitive, learning-driven detection that can observe traffic, reason about threats, and support adaptive mitigation across heterogeneous environments \cite{Ansari_Zerotrust}. In this setting, an intrusion detector is most useful as the learning core of a sense--infer--act process rather than a standalone classifier. Deep learning has shown promise for automated feature extraction \cite{Sayed_DDoS}; however, most methods treat traffic flows as independent feature vectors, discarding the relational structure that a cognitive detector could exploit. Furthermore, severe class imbalance, where benign traffic vastly outnumbers attack instances, degrades sensitivity to minority intrusion classes and remains a persistent challenge \cite{Shahraki_Net_Traffic}. Addressing structural modeling and distributional imbalance together is therefore a prerequisite for adaptive, robust detection and motivates the framework we develop here.

\subsection{Motivation}
The growing complexity of DDoS attacks and the limitations of traditional detection approaches necessitate new methods that capture the structural and distributional properties of modern network data. We highlight three key motivations for our graph-based adversarial framework.

\smallskip
\noindent\textbf{Non-Euclidean Structure in Network Traffic:} Network communications inherently exhibit graph-structured relationships through temporal sequences, source-destination interactions, and feature similarities that are often overlooked by conventional approaches \cite{bakar2024ftg}. Treating flows as independent samples discards relational information essential for recognizing coordinated and distributed attacks. Graph neural networks have recently demonstrated strong potential in network security by modeling such dependencies \cite{zhong2024survey}; however, their integration with adversarial augmentation remains largely unexplored. Our $k$-nearest neighbor ($k$-NN) graph construction captures flow interdependencies by forming temporal subgraphs that link flows by feature similarity within each time window, preserving both feature-space proximity and temporal locality.

\smallskip
\noindent\textbf{Real-World Class Imbalance:} Network datasets are highly imbalanced, with benign traffic overwhelmingly dominating attack instances, significantly reducing model sensitivity to minority intrusion classes \cite{qing2024mitigating}. Conventional interpolation-based oversampling techniques, such as SMOTE, fail to capture the nonlinear and topological dependencies inherent in high-dimensional traffic data \cite{chawla2002smote}. Although recent GAN-based augmentation methods have improved the generation of synthetic attacks \cite{zhao2023research}, they operate on vectorized representations and produce isolated samples without graph-structural context. To address this limitation, our adversarial framework employs a generator trained against a graph-based discriminator to synthesize realistic minority samples that preserve both statistical distributions and relational fidelity across network nodes, thereby enhancing robustness and detection performance.

\smallskip
\noindent\textbf{Topology-Aware Adversarial Training:} Many generative models synthesize isolated samples without incorporating relational context, resulting in data that may be statistically plausible yet topologically inconsistent \cite{arafah2025anomaly}. Even advanced generative architectures such as Wasserstein GANs \cite{arafah2025anomaly} and VAE-GAN hybrids \cite{tian2024vae} generate point-wise samples that neglect inter-flow dependencies. Our framework integrates graph neural networks (GNNs) with adversarial training to ensure that synthetic DDoS samples reflect both intrinsic flow characteristics and realistic relationships with neighboring flows, producing topology-consistent augmentation that existing generative approaches cannot achieve.

\subsection{Literature Review}
DDoS attacks have evolved from volumetric floods to multi-vector campaigns, exposing the limitations of traditional detection methods that treat flows independently and ignore relational structures \cite{zhong2024survey}. While machine learning improves on rule-based systems, most approaches rely on vectorized traffic representations, which miss topological patterns critical to distributed attacks. Combined with severe class imbalance as benign traffic dominates, this gap reduces detection accuracy and minority-class sensitivity.

\smallskip
\noindent\textbf{Feature Learning-based DDoS Detection.}
Chouhan {\itshape et al.} \cite{chouhan2023framework} proposed a non-linear kernel-based support vector machine for classifying DDoS attacks in software-defined networking (SDN). This lightweight model can detect live attacks but requires manual feature extraction. An LSTM-autoencoder-based approach \cite{el2022flow} reduces feature overhead and false alarms by addressing irrelevant features and non-representative SDN datasets, yet it still overlooks the inherent class-imbalance problem. To capture temporal information, a CNN-LSTM hybrid was introduced by Rajan and Aravindhar \cite{rajan2023detection}. On the ensemble learning front, He {\itshape et al.} \cite{he2024efficient} proposed a feature-weighted random forest for application-layer DDoS defense, achieving high detection rates but remaining limited to hand-crafted feature spaces. Yang {\itshape et al.} \cite{yang2024improved} combined attention mechanisms with BiGRU and Inception-CNN for IIoT intrusion detection, using mixed sampling to address class imbalance; however, the model still processes flows independently without structural context. Federated learning has also emerged as a promising direction: Li {\itshape et al.} \cite{li2021fleam} proposed an iterative model averaging (IMA) based gated recurrent unit (GRU) protocol at the fog/edge layer for collaborative and privacy-preserving DDoS mitigation in Industrial IoT networks. Despite these advances, the aforementioned methods fail to capture relational interdependencies within network traffic.

\smallskip
\noindent\textbf{Graph-based DDoS Attack Classification.} Graph networks have shown significant performance improvements in modeling data relationships by converting datasets into node-edge structures \cite{Nishanth_DDoS}. A comprehensive survey by Zhong {\itshape et al.} \cite{zhong2024survey} systematically categorized GNN-based IDS methods across graph construction strategies and deployment paradigms, identifying scalability and adversarial robustness as open challenges. Li {\itshape et al.} \cite{li2022graphddos} proposed a vanilla graph network to exploit relationships among packet attributes, time intervals, and device characteristics. A spatio-temporal GCN \cite{cao2021detecting} improved SDN-based DDoS detection by incorporating temporal information, though class imbalance remained unaddressed. Bakar {\itshape et al.} \cite{bakar2024ftg} proposed a similar approach combining multiple graph networks to extract features from non-Euclidean spaces. Lo {\itshape et al.} \cite{lo2022graphsage} introduced E-GraphSAGE, an edge-feature-aware extension of GraphSAGE that represents flows as graph edges with rich feature vectors, enabling joint learning from host and flow attributes for IoT intrusion detection. Wang {\itshape et al.} \cite{wang2025bs} proposed BS-GAT, a graph attention network that constructs behavioral similarity graphs with edge weights incorporated into the attention mechanism, achieving over 99\% accuracy in edge computing environments. Duan {\itshape et al.} \cite{duan2022application} proposed a graph model in a semi-supervised setting to mitigate data scarcity in DDoS detection. Although this approach improves detection performance, it still struggles with class imbalance.

\smallskip
\noindent\textbf{Adversarial Training for DDoS Detection.} Recent advances in deep generative models have demonstrated the potential of adversarial training for DDoS detection \cite{mustapha2023detecting}. Shieh {\itshape et al.} \cite{shieh2022detection} proposed a dual-discriminator strategy for detecting adversarial DDoS traffic in SDN. A GAN model \cite{mustapha2023detecting} was introduced to generate synthetic data for highly imbalanced datasets, improving classification accuracy by reframing DDoS detection as a supervised learning task. Zhao {\itshape et al.} \cite{zhao2023research} employed a conditional GAN to oversample rare attack categories, preventing category omission in heavily skewed datasets through label-consistent synthetic generation. Arafah {\itshape et al.} \cite{arafah2025anomaly} combined a denoising autoencoder with a Wasserstein GAN, using the WGAN's Lipschitz-constrained training to synthesize more stable and realistic minority-class traffic. Tian {\itshape et al.} \cite{tian2024vae} extended this line by integrating a variational autoencoder with an auxiliary-classifier WGAN-GP, thereby improving both the fidelity and the class-conditioned diversity of generated intrusion samples. Despite these advances, none of these generative approaches operate on graph-structured data, producing isolated synthetic samples that lack topological context.

\smallskip
\noindent\textbf{Learning-Driven Cognitive Network Defense.} A parallel line of work embeds learning inside the network control plane so that defense adapts online rather than through static rules. Collaborative and decentralized learning has been explored for 5G+ core and edge environments \cite{hossain2023decentralized}, and meta- and resource-aware learning has been applied to offloading and control in IoT networks \cite{hossain_metaRL}. Federated designs push detection to the fog/edge layer for privacy-preserving, in-network mitigation \cite{hossain_metaRL}, while flow-level detectors have been deployed at SDN controllers to close the loop between sensing and actuation \cite{el2022flow}. These efforts frame intrusion detection as one stage of a cognitive networking pipeline, but they largely operate on vectorized traffic and do not exploit the relational topology of coordinated attacks. GraphGAN targets exactly this gap: it supplies a topology-aware cognitive engine whose output can drive an edge/SDN mitigation decision, uniting structural modeling with the adaptive control that cognitive communication networks require.

\subsection{Contributions}
To address these challenges, we propose a Graph-based Generative Adversarial Network (GraphGAN), a novel framework that unifies structure-aware modeling with generative augmentation for intelligent DDoS detection. Unlike prior work, our approach treats traffic as inherently relational and leverages adversarial training to mitigate class imbalance while preserving topological integrity. The main contributions are summarized as follows:
\begin{itemize}
    \item A structure-aware graph construction strategy using $k$-NN connectivity over sliding temporal windows to capture flow co-occurrence and feature-similarity patterns within temporal neighborhoods that are ignored by vector-based detection approaches.
    \item An adversarial graph neural architecture integrating GCNs with generative training, where the discriminator processes entire graph structures, enabling topology-conditioned synthetic sample generation over fixed real topology templates.
    \item An imbalance-aware generative mechanism that produces realistic minority-class DDoS samples while preserving both statistical fidelity and relational consistency within the graph.
    \item A detection framework based on GCN classification that exploits structural flow relationships to identify coordinated attack patterns.
\end{itemize}

The remainder of this manuscript is organized as follows. Section~\ref{method} introduces the proposed GraphGAN framework. Section~\ref{results} presents the experimental results and performance evaluation. Section~\ref{conclusion} includes the concluding remarks.

\section{Methodology}\label{method}
The proposed GraphGAN integrates graph neural networks with adversarial training to address class imbalance while modeling topological relationships in network traffic. The framework is illustrated in Fig.~\ref{fig:graph_gan}.

\subsection{Cognitive Detection Process}
\label{subsec:sysmodel}
We frame GraphGAN as the cognitive detection stage of a \emph{sense--infer--act} process that a network controller (e.g., an SDN controller) can use to maintain service availability under attack. In the \emph{sense} stage, per-flow statistics $\mathbf{x}_i \in \mathbb{R}^d$ are exported from switches or gateways (e.g., via telemetry or flow-export) and streamed to the detector. In the \emph{infer} stage, the proposed GraphGAN pipeline converts a sliding window of flows into a temporal graph $\mathcal{G}_t$ and produces a subgraph-level decision $P(y\mid\mathcal{G}_t)$; adversarial augmentation supplies the minority-class fidelity that keeps this decision reliable under the imbalance typical of live traffic. In the \emph{act} stage, a decision of ``DDoS'' can trigger a mitigation policy (e.g., rate limiting or flow-rule installation), while benign decisions leave forwarding untouched.

Two properties make this framing attractive. First, the detector consumes only exported flow features, so it imposes no per-packet inference cost on the forwarding path. Second, because the decision is made at the granularity of a temporal window rather than a single flow, detection reacts to \emph{coordinated} behavior---the regime in which distributed attacks are distinguishable (Theorem~\ref{thm:advantage})---rather than to isolated flows.

\subsection{Graph Construction from Network Traffic}
Our graph construction methodology transforms sequential network flow data into graph-structured representations that preserve both feature-level similarities and temporal dependencies.

\subsubsection{Feature Space Representation and Normalization}
Let $\mathcal{D}=\{(\mathbf{x}_i,y_i)\}_{i=1}^{N}$ denote the traffic dataset, where $\mathbf{x}_i \in \mathbb{R}^{d}$ is the $d$-dimensional feature vector of the $i$-th flow and $y_i \in \{0,1\}$ is the binary label (benign or DDoS). Each $\mathbf{x}_i$ consists of flow-level statistics, including packet size, inter-arrival time, flow duration, and protocol-related features. To avoid dominance of large-magnitude features and, crucially, to confine every feature to a bounded range that the generator can reproduce, we standardize each feature to the interval $[-1,1]$ via min--max normalization,
\begin{equation}\label{eq:standardize}
\tilde{\mathbf{x}}_i = 2\,\frac{\mathbf{x}_i - \mathbf{x}_{\min}}{\mathbf{x}_{\max} - \mathbf{x}_{\min}} - \mathbf{1},
\end{equation}
where $\mathbf{x}_{\min}$ and $\mathbf{x}_{\max}$ are the element-wise minimum and maximum computed over the training dataset and all operations are applied element-wise. Beyond balancing feature contributions in graph construction, bounding features to $[-1,1]$ aligns their support with the codomain of the generator's $\tanh$ output, a property we exploit to keep the adversarial distribution-matching objective (Theorem~\ref{thm:gan_convergence}) well-posed.

\subsubsection{Temporal Graph Construction via Sliding Windows}
Network attacks, particularly DDoS attacks, exhibit temporal patterns crucial for accurate detection. To capture these patterns while maintaining computational tractability, we employ a sliding-window mechanism that partitions sequential network flows into temporal subgraphs. For window size $w$ and step size $s$, we define temporal subgraphs
$\mathcal{G}_t = (\mathcal{V}_t, \mathcal{E}_t, \mathbf{X}_t),$
where $\mathcal{V}_t = \{v_{t,1}, \ldots, v_{t,w}\}$ is the vertex set of $w$ consecutive flows, $\mathcal{E}_t$ represents the edge set encoding relational dependencies among flows, and $\mathbf{X}_t \in \mathbb{R}^{w \times d}$ is the standardized feature matrix. The sliding window ensures temporal locality, and overlapping windows ($s < w$) capture transitional patterns for robust representation learning. Because overlapping windows share raw flows, windows are never allowed to span across data partitions: the chronological splitting protocol used in our experiments (Section~\ref{results}) partitions the flow stream before windowing, preventing any flow from leaking across the training, validation, and test sets.

\begin{figure*}[!t]
    \centering
    \includegraphics[width=0.8\textwidth]{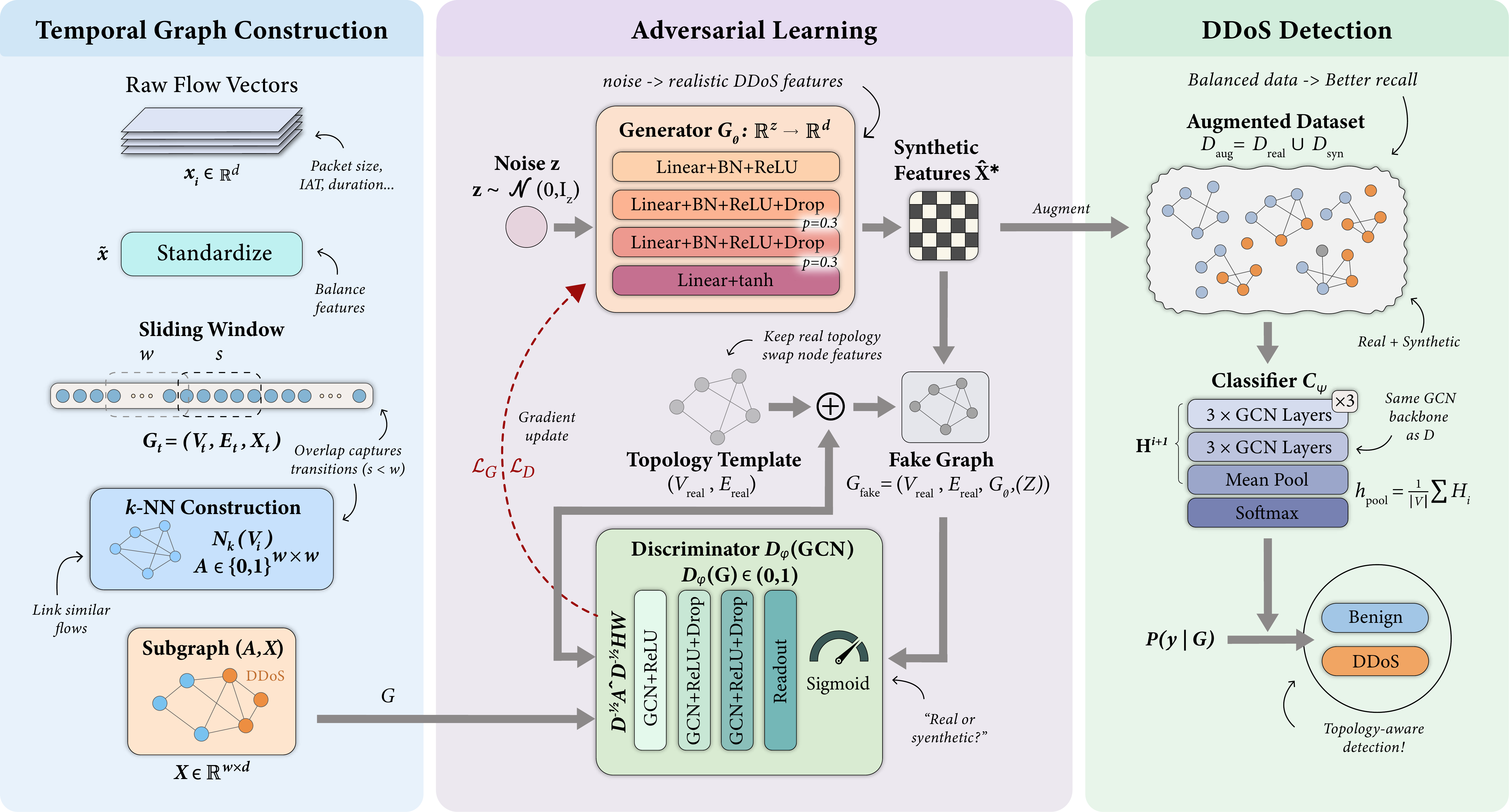}
    \caption{Overview of the proposed GraphGAN framework. (Left) Raw flow vectors are segmented using sliding windows and transformed into $k$-NN subgraphs. (Center) The adversarial training loop optimizes a generator to synthesize realistic node features over fixed topology templates, while a GCN-based discriminator distinguishes real from synthetic graphs. (Right) A GCN classifier operates on the augmented dataset to perform binary DDoS classification.}
    \label{fig:graph_gan}
\end{figure*}

\subsubsection{\textit{k}-Nearest Neighbor Graph Construction}
Constructing meaningful graph topology is critical for capturing the relational structure inherent in network traffic. We employ $k$-nearest neighbor ($k$-NN) graph construction based on feature similarity, creating edges between network flows with similar characteristics. The $k$-NN graph construction begins by computing the pairwise Euclidean distance matrix for all flows within each temporal window:
\begin{equation}
D_{ij}^{(t)} = \|\tilde{\mathbf{x}}_{t,i} - \tilde{\mathbf{x}}_{t,j}\|_2 = \sqrt{\sum_{\ell=1}^{d}(\tilde{x}_{t,i,\ell} - \tilde{x}_{t,j,\ell})^2}
\end{equation}

where $\tilde{\mathbf{x}}_{t,i}$ represents the standardized feature vector for the $i$-th node in temporal window $t$. The Euclidean distance metric is chosen for its interpretability and computational efficiency, though other distance metrics could be employed depending on the feature characteristics. For each node $v_i$ in the temporal window, we identify its $k$ nearest neighbors:
\begin{equation}
\mathcal{N}_k(v_i) = \{v_j : j \in \text{argmin}_{j \neq i}^k D_{ij}^{(t)}\}
\end{equation}

This operation selects the $k$ nodes most similar to node $v_i$ in the feature space, excluding $ v_i$ itself. The parameter $k$ controls the graph's connectivity and must be chosen to balance capturing local neighborhoods (small $k$) and ensuring connectivity (large $k$). The adjacency matrix $\mathbf{A}_t \in \{0,1\}^{w \times w}$ is then constructed to create an undirected graph:
\begin{equation}\label{eq:adj}
A_{ij}^{(t)} = \begin{cases}
1, & \text{if } v_j \in \mathcal{N}_k(v_i) \text{ or } v_i \in \mathcal{N}_k(v_j), \\
0, & \text{otherwise}.
\end{cases}
\end{equation}

The symmetric construction ensures that if node $v_i$ considers $v_j$ a neighbor, then $v_j$ also considers $v_i$ a neighbor, yielding an undirected graph that better captures bidirectional relationships in network communications. Two important structural properties follow directly from this symmetric $k$-NN construction: the adjacency matrix is guaranteed to be symmetric, and every node has a degree bounded between $k$ and $w{-}1$. We formalize these properties below.

\begin{lemma}[Symmetry and Bounded Degree of $k$-NN Graph]
\label{lem:knn_properties}
The adjacency matrix $\mathbf{A}_t$ constructed via symmetric $k$-NN satisfies: (i) $\mathbf{A}_t = \mathbf{A}_t^\top$ (symmetry), and (ii) for every node $v_i$, the degree $\deg(v_i) = \sum_{j} A_{ij}^{(t)}$ satisfies $k \leq \deg(v_i) \leq w-1$, where $w$ is the window size.
\end{lemma}

\emph{Proof}
Property (i) follows directly from the construction rule in Eq.~\eqref{eq:adj}: $A_{ij}^{(t)} = 1$ if and only if $v_j \in \mathcal{N}_k(v_i)$ \emph{or} $v_i \in \mathcal{N}_k(v_j)$, which is symmetric by definition. For property (ii), each node $v_i$ selects exactly $k$ neighbors, so $\deg(v_i) \geq k$. Since the graph contains $w$ nodes and self-loops are excluded, the degree is bounded above by $w-1$.

\noindent This bounded-degree property ensures that each node maintains at least $k$ connections while remaining constrained by the finite window size $w$, preventing degenerate graph structures and providing a well-conditioned topology for subsequent GCN processing.

\subsubsection{Graph Label Assignment}
Each temporal subgraph requires a label for supervised learning. Since individual network flows within a window may have different labels, we employ majority voting to assign a single label to the entire subgraph as:

\begin{equation}\label{eq:label}
y_t = \text{argmax}_{c \in \{0,1\}} \sum_{i=1}^{w} \mathbb{I}(y_{t,i} = c),
\end{equation}
where $\mathbb{I}(\cdot)$ is the indicator function that returns 1 if the condition is true and 0 otherwise, and $y_{t,i}$ is the label of the $i$-th flow in window $t$. For even window sizes, an exact tie between the two classes is possible; we break such ties toward the attack (DDoS) class, a security-conservative choice that favors recall over precision. This approach assumes that the majority class within a temporal window reflects the overall behavior pattern, which is reasonable for DDoS attacks, which typically involve sustained malicious activity across multiple flows. A natural question is whether this majority-voting scheme reliably assigns labels for the window sizes used in our experiments, including the even default $w=30$. The following lemma confirms that the assigned label is always supported by at least a simple majority, which strengthens to a strict majority for odd window sizes.

\begin{lemma}[Label Consistency of Majority Voting]
\label{lem:label_consistency}
Let $\alpha_t = \frac{1}{w}\sum_{i=1}^{w} \mathbb{I}(y_{t,i} = y_t)$ denote the fraction of flows in window $t$ agreeing with the label $y_t$ assigned by Eq.~\eqref{eq:label}, with ties broken toward the attack class. Then, $\alpha_t \geq \frac{1}{2}$ for any window size $w$. Moreover, when $w$ is odd, tie is not allowed and the bound strengthens to a strict majority, $\alpha_t \geq \frac{1}{2} + \frac{1}{2w}$. Under temporal locality, where DDoS flows form contiguous segments so that most windows lie within a single regime, $\alpha_t \to 1$; the worst case $\alpha_t = \frac{1}{2}$ arises only at an exact class split within a window.
\end{lemma}

\emph{Proof}
For binary labels, one of the two classes must account for at least $\lceil w/2 \rceil$ of the $w$ flows, so the majority-voted winner satisfies $\sum_{i=1}^{w}\mathbb{I}(y_{t,i} = y_t) \geq \lceil w/2 \rceil$; the tie-breaking rule guarantees this even at an exact $w/2$ split. Dividing by $w$ yields $\alpha_t \geq \lceil w/2\rceil / w \geq \frac{1}{2}$. When $w$ is odd, the class counts $\sum_{i}\mathbb{I}(y_{t,i}=1)$ and $w-\sum_{i}\mathbb{I}(y_{t,i}=1)$ have opposite parities and cannot be equal, so tie is not possible and the winning count is at least $\frac{w+1}{2}$. Dividing by $w$ gives $\alpha_t \geq \frac{w+1}{2w} = \frac{1}{2} + \frac{1}{2w}$, which establishes the strict-majority case.

\subsubsection{When Graph Structure Carries Information}
We formalize two properties of DDoS traffic that together establish when graph structure provides a detection advantage over single-flow models. The first says individual flows are \emph{ambiguous}; the second says they are \emph{jointly informative}. Both are necessary for structure to help, and we state them as explicit hypotheses.

\begin{definition}[Feature Ambiguity]
\label{def:ambig}
Flows are \emph{$\eta$-ambiguous} if the optimal single-flow (vector) classifier has irreducible Bayes risk
$R^{\star}_{\mathrm{flow}}
   = \mathbb{E}_{\mathbf{x}}\!\left[\min\!\big(p(y{=}0\mid\mathbf{x}),\,
     p(y{=}1\mid\mathbf{x})\big)\right] = \eta > 0.$
\end{definition}

\begin{definition}[Coordination]
\label{def:coord}
DDoS traffic is \emph{coordinated} if, for flows $i,j$ connected by the $k$-NN rule, the labels retain dependence after conditioning on their features:
$I\!\left(Y_i ; Y_j \mid \mathbf{X}_i, \mathbf{X}_j\right) = c > 0.$
\end{definition}

\noindent Definition~\ref{def:coord} captures the intuition that botnet flows are not independent draws: knowing a neighbor is malicious raises the posterior that a feature-similar flow is malicious too, \emph{beyond} what the features alone reveal. We now show that a neighborhood-aware predictor strictly beats the best single-flow predictor exactly in this regime.

\begin{theorem}[Graph Advantage under Feature-Ambiguous Coordination]
\label{thm:advantage}
Assume flows are $\eta$-ambiguous (Definition~\ref{def:ambig}) and coordinated with strength $c>0$ (Definition~\ref{def:coord}). Let $R^{\star}_{\mathrm{graph}}$ be the Bayes risk of a predictor that observes a node together with its $k$-NN neighborhood within the window. Then,
\[
R^{\star}_{\mathrm{flow}} - R^{\star}_{\mathrm{graph}} \;\ge\; \Phi(c,\bar{m}) \;>\; 0,
\]
where $\bar{m}$ is the expected number of same-class neighbors and $\Phi$ is strictly increasing in both arguments, with $\Phi(0,\cdot)=0$.
\end{theorem}

\emph{Proof}
The neighborhood supplies side information $Z_i = (Y_{j})_{j\in\mathcal{N}_k(v_i)}$ about $Y_i$. By Definition~\ref{def:coord}, $I(Y_i; Z_i \mid \mathbf{X}_i) \ge c\,\bar m > 0$, so the conditional entropy strictly drops:
$H(Y_i\mid \mathbf{X}_i, Z_i) < H(Y_i\mid \mathbf{X}_i)$.
Fano's inequality lower-bounds any predictor's risk by an increasing function of the conditional entropy of the target; a strict entropy reduction therefore admits a neighborhood predictor whose Bayes risk is strictly smaller, with the gap controlled by the mutual information drop. Setting $\Phi$ to the induced risk decrease gives the claim; $\Phi(0,\cdot)=0$ because $c=0$ removes the side information.

\noindent\textbf{Scope.} The hypothesis $\eta>0$ is essential: if flows are already separable in feature space, then $R^{\star}_{\mathrm{flow}}=0$ and no model can improve on it, so the advantage vanishes. The theorem, therefore, characterizes \emph{when} structure helps (ambiguous yet coordinated traffic) rather than asserting structure is universally required. This is the regime that DDoS occupies, and that motivates the entire framework.

The advantage above lies in the dependence between features and the topology they induce. Because the $k$-NN adjacency is a deterministic function of the window's features, this dependence is real and, crucially, invisible to any feature-only model.

\begin{proposition}[Feature--Topology Information Gap]
\label{prop:gap}
Let a vector or vectorized-GAN model act on $\mathbf{X}$ alone, i.e.\ its predictor and generator are measurable with respect to features only. Then, no such model can exploit the feature--topology mutual information $I_{\mathrm{data}}(\mathbf{X};\mathbf{A})$. A model that processes $(\mathbf{X},\mathbf{A})$ jointly, as our GCN does, has access to this term.
\end{proposition}

\emph{Proof}
A feature-only predictor factors through $\mathbf{X}$, so by the data-processing inequality, its accessible information about $Y$ is at most $I(\mathbf{X};Y)$. The joint predictor's accessible information is $I(\mathbf{X},\mathbf{A};Y) = I(\mathbf{X};Y) + I(\mathbf{A};Y\mid\mathbf{X})$, and the second term is non-zero precisely when topology is label-informative given features---the coordinated regime of Definition~\ref{def:coord}.

\noindent The quantity $I_{\mathrm{data}}(\mathbf{X};\mathbf{A})$ reappears as the irreducible residual in the adversarial augmentation analysis (Theorem~\ref{thm:residual}) and connects to the generalization bound of Theorem~\ref{thm:gen}, forming the information-theoretic thread that links graph construction, adversarial training, and classification throughout the framework.

\subsection{Adversarial Learning Framework}
By learning the underlying distribution of minority class samples (DDoS attacks), our GAN-based approach generates realistic synthetic samples that preserve both statistical properties and topological relationships.

\subsubsection{Generator Network Architecture}
The generator network $G_\theta: \mathbb{R}^{z} \rightarrow \mathbb{R}^{d}$ transforms random noise vectors sampled from a prior distribution into synthetic node features that mimic the characteristics of real DDoS traffic. The architecture employs multiple fully connected layers with batch normalization and dropout for regularization. Given a noise vector $\mathbf{z} \sim \mathcal{N}(\mathbf{0}, \mathbf{I}_z)$ sampled from a multivariate standard normal distribution, where $z$ is the noise dimension, the generator produces synthetic features through the following sequence of transformations:
\begin{align}
\mathbf{h}_1 &= \text{ReLU}(\text{BN}(\mathbf{W}_1^G \mathbf{z} + \mathbf{b}_1^G)), \\
\mathbf{h}_2 &= \text{Dropout}(\text{ReLU}(\text{BN}(\mathbf{W}_2^G \mathbf{h}_1 + \mathbf{b}_2^G)), p=0.3), \\
\mathbf{h}_3 &= \text{Dropout}(\text{ReLU}(\text{BN}(\mathbf{W}_3^G \mathbf{h}_2 + \mathbf{b}_3^G)), p=0.3), \\
\hat{\mathbf{x}} &= \tanh(\mathbf{W}_4^G \mathbf{h}_3 + \mathbf{b}_4^G),
\end{align}

where $\mathbf{W}_i^G \in \mathbb{R}^{h_i \times h_{i-1}}$ and $\mathbf{b}_i^G \in \mathbb{R}^{h_i}$ are the weight matrices and bias vectors for the $i$-th layer, respectively, with $h_0 = z$ and $h_4 = d$. The batch normalization operation $\text{BN}(\cdot)$ normalizes the pre-activation values to have zero mean and unit variance, accelerating training and improving stability. The ReLU activation function $\text{ReLU}(x) = \max(0, x)$ introduces non-linearity while maintaining computational efficiency. Dropout with probability $p=0.3$ randomly sets a fraction of input units to zero during training, preventing overfitting. The final activation $\tanh(x) = \frac{e^x - e^{-x}}{e^x + e^{-x}}$ bounds each synthetic feature to $[-1, 1]$, stabilizing adversarial training and, by construction, matching the support of the real node features that were min--max normalized to the same interval in Eq.~\eqref{eq:standardize}. This shared support is essential: were the synthetic and real features confined to different ranges, the discriminator could separate them by feature range alone, and the Jensen--Shannon optimum of Theorem~\ref{thm:gan_convergence} (attained at $p_G = p_{\text{data}}$) could never be reached.

\subsubsection{Graph-based Discriminator Network}
The discriminator network $D_\phi$ employs Graph Convolutional Networks (GCNs) to process graph-structured data and distinguish between real and synthetic samples. Unlike traditional discriminators that operate on individual samples, our graph-based approach considers the entire graph topology, enabling more sophisticated pattern recognition. For a graph $\mathcal{G} = (\mathcal{V}, \mathcal{E}, \mathbf{X})$, each GCN layer performs localized convolution operations that aggregate information from neighboring nodes. The fundamental GCN operation is defined as:
\begin{equation}
\mathbf{H}^{(l+1)} = \sigma\!\left(\hat{\mathbf{D}}^{-\frac{1}{2}}\hat{\mathbf{A}}\hat{\mathbf{D}}^{-\frac{1}{2}}\mathbf{H}^{(l)}\mathbf{W}^{(l)}\right),
\end{equation}

where $\hat{\mathbf{A}} = \mathbf{A} + \mathbf{I}$ is the adjacency matrix with added self-loops to include each node's own features in the aggregation, $\hat{\mathbf{D}}_{ii} = \sum_j \hat{A}_{ij}$ is the corresponding degree matrix, $\mathbf{H}^{(l)} \in \mathbb{R}^{|\mathcal{V}| \times h^{(l)}}$ denotes the node representations at layer $l$ with $\mathbf{H}^{(0)} = \mathbf{X}$, $\mathbf{W}^{(l)} \in \mathbb{R}^{h^{(l)} \times h^{(l+1)}}$ is the trainable weight matrix, and $\sigma(\cdot)$ is the activation function. The normalization term $\hat{\mathbf{D}}^{-\frac{1}{2}}\hat{\mathbf{A}}\hat{\mathbf{D}}^{-\frac{1}{2}}$ ensures degree-normalized feature propagation, preventing high-degree nodes from disproportionately influencing the learned representations. This propagation rule arises as a first-order Chebyshev approximation of spectral graph convolutions \cite{defferrard2016convolutional}.

\subsubsection{Adversarial Training Objective}
Adversarial training follows a two-player minimax game in which the generator attempts to fool the discriminator, while the discriminator tries to correctly identify real and synthetic samples. The minimax objective is formulated as:
\begin{equation}
\min_{\theta} \max_{\phi} \mathcal{L}_{\text{GAN}}(G_\theta, D_\phi),
\end{equation}
where $\theta$ and $\phi$ represent the parameters of the generator and discriminator, respectively. The objective function captures the adversarial nature of the training as
$\mathcal{L}_{\text{GAN}}(G_\theta, D_\phi) = \mathbb{E}_{\mathcal{G} \sim p_{\text{data}}}[\log D_\phi(\mathcal{G})] + \mathbb{E}_{\mathcal{G} \sim p_G}[\log(1 - D_\phi(\mathcal{G}))].$
Here, $p_{\text{data}}$ denotes the distribution of \emph{real minority-class} (DDoS) subgraphs, since the augmentation targets only the under-represented attack class and benign graphs are never synthesized, while $p_G$ is the generator-induced distribution over synthetic DDoS graphs. The first term encourages the discriminator to correctly identify real graphs by maximizing $\log D_\phi(\mathcal{G})$, while the second term encourages correct identification of fake graphs by maximizing $\log(1 - D_\phi(\mathcal{G}))$. From the generator's perspective, minimizing this objective encourages it to increase the discriminator's misclassification probability on synthetic samples. The construction of fake graphs preserves the topological structure of real graphs while replacing node features with generator outputs,
$\mathcal{G}_{\text{fake}} = (\mathcal{V}_{\text{real}}, \mathcal{E}_{\text{real}}, G_\theta(\mathbf{Z})),$
where $\mathbf{Z} = [\mathbf{z}_1, \mathbf{z}_2, \ldots, \mathbf{z}_{|\mathcal{V}|}]^T$ is a matrix of independently sampled noise vectors. This approach ensures that synthetic graphs maintain realistic topological properties while introducing novel feature combinations. The individual loss functions for training are:
\begin{align}
\mathcal{L}_D &= -\mathbb{E}_{\mathcal{G} \sim p_{\text{data}}}[\log D_\phi(\mathcal{G})] - \mathbb{E}_{\mathcal{G} \sim p_G}[\log(1 - D_\phi(\mathcal{G}))], \\
\mathcal{L}_G &= -\mathbb{E}_{\mathcal{G} \sim p_G}[\log D_\phi(\mathcal{G})].
\end{align}
The discriminator loss $\mathcal{L}_D$ is minimized to improve classification accuracy on both real and fake samples, while the generator loss $\mathcal{L}_G$ is minimized to increase the probability that synthetic samples are classified as real, viz., to produce more realistic synthetic samples. Having defined the adversarial objective, we now characterize the equilibrium behavior of this minimax game. We first derive the closed-form optimal discriminator for a fixed generator, and then show that the global optimum is achieved when the generator perfectly recovers the real data distribution.

\begin{theorem}[Optimal Discriminator]
\label{thm:optimal_disc}
For a fixed generator $G_\theta$, the optimal discriminator $D_\phi^*$ is given by
\begin{equation}
D_\phi^*(\mathcal{G}) = \frac{p_{\text{data}}(\mathcal{G})}{p_{\text{data}}(\mathcal{G}) + p_G(\mathcal{G})},
\end{equation}
where $p_{\text{data}}$ is the distribution of real graph data and $p_G$ is the distribution induced by the generator \cite{goodfellow2014generative}.
\end{theorem}
\emph{Proof}
For fixed $G_\theta$, the discriminator objective $\mathcal{L}_{\text{GAN}}$ can be written as
\begin{equation}
\int \Big[ p_{\text{data}}(\mathcal{G}) \log D(\mathcal{G}) + p_G(\mathcal{G}) \log\!\big(1 - D(\mathcal{G})\big) \Big] d\mathcal{G}.
\end{equation}
For any $(a, b)$ with $a, b \geq 0$ and $a + b > 0$, the function $f(D) = a\log D + b\log(1-D)$ attains its maximum at $D^* = \frac{a}{a+b}$, obtained by setting $\frac{df}{dD} = \frac{a}{D} - \frac{b}{1-D} = 0$. Substituting $a = p_{\text{data}}(\mathcal{G})$ and $b = p_G(\mathcal{G})$ yields the result.

\noindent Building on this result, we can now analyze the global optimum of the minimax game by substituting the optimal discriminator back into the objective function.

\begin{theorem}[Global Optimality of Adversarial Training]
\label{thm:gan_convergence}
The global minimum of the minimax objective $\mathcal{L}_{\text{GAN}}(G_\theta, D_\phi)$ is achieved if and only if $p_G = p_{\text{data}}$. At this optimum, $D_\phi^*(\mathcal{G}) = \frac{1}{2}$ for all $\mathcal{G}$, and $\min_\theta \max_\phi \mathcal{L}_{\text{GAN}} = -\log 4$ \cite{goodfellow2014generative}.
\end{theorem}
\emph{Proof}
Substituting $D^*$ from Theorem~\ref{thm:optimal_disc} into the objective yields
\begin{equation}
\begin{aligned}
C(G) &= \mathbb{E}_{\mathcal{G} \sim p_{\text{data}}}\!\left[\log \frac{p_{\text{data}}(\mathcal{G})}{p_{\text{data}}(\mathcal{G}) + p_G(\mathcal{G})}\right] \\
&\quad+ \mathbb{E}_{\mathcal{G} \sim p_G}\!\left[\log \frac{p_G(\mathcal{G})}{p_{\text{data}}(\mathcal{G}) + p_G(\mathcal{G})}\right].
\end{aligned}
\end{equation}
This can be rewritten as
\begin{equation}
\begin{aligned}
C(G) &= -\log 4 + \mathrm{KL}\!\left(p_{\text{data}} \Big\| \tfrac{p_{\text{data}} + p_G}{2}\right) \\
&\quad+ \mathrm{KL}\!\left(p_G \Big\| \tfrac{p_{\text{data}} + p_G}{2}\right) \\
&= -\log 4 + 2 \cdot \mathrm{JSD}(p_{\text{data}} \| p_G),
\end{aligned}
\end{equation}
where $\mathrm{JSD}$ denotes the Jensen--Shannon divergence. Since $\mathrm{JSD} \geq 0$ with equality if and only if $p_G = p_{\text{data}}$, the global minimum $C(G) = -\log 4$ is attained uniquely when $p_G = p_{\text{data}}$.

\noindent Theorems~\ref{thm:optimal_disc} and~\ref{thm:gan_convergence} provide the theoretical foundation for our adversarial framework: when the generator successfully learns the distribution of DDoS attack graphs, the discriminator can no longer distinguish synthetic from real attack traffic, ensuring high-fidelity augmentation for downstream classification.

Because our generator produces features independently of the template topology, the equilibrium of Theorem~\ref{thm:gan_convergence} does not guarantee that the joint distribution over features and adjacency is matched. The following result identifies the irreducible gap.

\begin{theorem}[Residual Divergence of Topology-Independent Generation]
\label{thm:residual}
Because templates are real and features are generated independently of them, the generator induces $p_{\mathrm{fake}}(\mathbf{X},\mathbf{A}) = p_G(\mathbf{X})\,p_{\mathrm{data}}(\mathbf{A})$. Here, $\mathrm{KL}(\cdot\|\cdot)$ denotes the Kullback--Leibler divergence \cite{kullback1951information}. Then,
\[
\min_{\theta}\,\mathrm{KL}\!\left(p_{\mathrm{data}}\,\|\,p_{\mathrm{fake}}\right)
\text{ is attained at } p_G(\mathbf{X}) = p_{\mathrm{data}}(\mathbf{X}),
\]
and the irreducible value equals the feature--topology mutual information,
$\mathrm{KL}\!\left(p_{\mathrm{data}}\,\|\,p_{\mathrm{data}}(\mathbf{X})\,p_{\mathrm{data}}(\mathbf{A})\right)
   = I_{\mathrm{data}}(\mathbf{X};\mathbf{A}).$
\end{theorem}

\emph{Proof}
By the chain rule of KL divergence and the factorized form of $p_{\mathrm{fake}}$,
$\mathrm{KL}(p_{\mathrm{data}}\|p_{\mathrm{fake}})
 = \mathbb{E}_{\mathbf{A}}\!\left[\mathrm{KL}\big(p_{\mathrm{data}}(\mathbf{X}\mid\mathbf{A})\,\|\,p_G(\mathbf{X})\big)\right]$.
This is minimized over $p_G$ at the $\mathbf{A}$-averaged conditional $p_G(\mathbf{X})=\mathbb{E}_{\mathbf{A}}[p_{\mathrm{data}}(\mathbf{X}\mid\mathbf{A})]=p_{\mathrm{data}}(\mathbf{X})$. Substituting back leaves exactly the mutual information between $\mathbf{X}$ and $\mathbf{A}$.

\noindent Theorem~\ref{thm:residual} closes the loop with Proposition~\ref{prop:gap}: the single quantity $I_{\mathrm{data}}(\mathbf{X};\mathbf{A})$ is both what vector methods miss and what a topology-independent generator cannot reproduce.

\begin{corollary}[Topology-Stable Consistency]
\label{cor:stable}
If generated features are topology-stable---the $k$-NN graph they induce matches the template up to bounded edge-edit distance---then $I_{\mathrm{data}}(\mathbf{X};\mathbf{A})$ contributed by the synthetic samples vanishes and augmentation is asymptotically consistent. This condition is the precise sense in which a future topology-conditioned generator would be strictly more expressive.
\end{corollary}

\subsection{Graph-based Classification Network}
The classification network serves as the final component of our framework, leveraging both original and synthetically augmented data to perform DDoS detection.

\subsubsection{Classifier Architecture and Feature Learning}
The DDoS classifier $C_\psi$ employs the same GCN-based architecture as the discriminator to ensure consistent graph-structured processing, but it is trained independently with a classification objective rather than an adversarial one, allowing it to focus specifically on distinguishing benign from malicious traffic patterns. The classifier processes input graphs through multiple GCN layers that progressively refine node representations:
\begin{align}
\mathbf{H}_C^{(1)} &= \text{ReLU}(\text{GCN}(\mathbf{X}, \mathbf{A})), \\
\mathbf{H}_C^{(2)} &= \text{Dropout}(\text{ReLU}(\text{GCN}(\mathbf{H}_C^{(1)}, \mathbf{A})), p=0.3), \\
\mathbf{H}_C^{(3)} &= \text{Dropout}(\text{ReLU}(\text{GCN}(\mathbf{H}_C^{(2)}, \mathbf{A})), p=0.3), \\
\mathbf{h}_{C,\text{pool}} &= \frac{1}{|\mathcal{V}|}\sum_{i=1}^{|\mathcal{V}|} \mathbf{H}^{(3)}_{C,\,i,:}, \\
P(y \mid \mathcal{G}) &= \text{Softmax}(\mathbf{W}_C \mathbf{h}_{C,\text{pool}} + \mathbf{b}_C).
\end{align}
The softmax activation function in the final layer converts the output logits into a probability distribution over the two classes:
\begin{equation}
\text{Softmax}(\mathbf{o})_i = \frac{\exp(o_i)}{\sum_{j=1}^{2}\exp(o_j)},
\end{equation}
where $\mathbf{o} = \mathbf{W}_C\mathbf{h}_{C,\text{pool}} + \mathbf{b}_C$ represents the pre-activation logits. This formulation ensures that $\sum_{i=1}^{2}P(y=i|\mathcal{G}) = 1$ and enables probabilistic interpretation of the classification results. An important design consideration for the GCN classifier is its receptive field, i.e., how far information propagates across multiple graph layers. The following proposition quantifies the receptive field size in terms of the number of GCN layers and the graph's bounded degree established in Lemma~\ref{lem:knn_properties}.

\begin{proposition}[Receptive Field of $L$-Layer GCN Classifier]
\label{prop:receptive_field}
In an $L$-layer GCN, the representation $\mathbf{h}_i^{(L)}$ of node $v_i$ aggregates information from all nodes within its $L$-hop neighborhood $\mathcal{N}^{(L)}(v_i) = \{v_j : d_{\mathcal{G}}(v_i, v_j) \leq L\}$, where $d_{\mathcal{G}}$ denotes the shortest-path distance on $\mathcal{G}$. For our $k$-NN graphs with maximum degree $\Delta \leq w-1$ (Lemma~\ref{lem:knn_properties}), the receptive field size is bounded by $|\mathcal{N}^{(L)}(v_i)| \leq \min\!\left(\sum_{\ell=0}^{L} \Delta^\ell,\, w\right)$ \cite{xu2018powerful}.
\end{proposition}

\noindent With $L=3$ layers, $k=5$, and window size $w=30$, the symmetric OR-construction of Eq.~\eqref{eq:adj} yields an expected node degree on the order of $2k=10$, so the receptive field saturates the window within two hops: $|\mathcal{N}^{(L)}(v_i)| \leq \min\!\left(\sum_{\ell=0}^{3} \Delta^\ell,\, w\right) = w = 30$. Each node's representation, therefore, incorporates information from the entire temporal window, allowing the classifier to capture global attack coordination patterns within each subgraph while avoiding the over-smoothing that occurs with deeper architectures.

The bounded degree of Lemma~\ref{lem:knn_properties} also yields a certified radius within which classifier predictions cannot flip, providing a formal robustness guarantee for the GCN architecture.

\begin{proposition}[Stability under Feature and Topology Perturbation]
\label{prop:robust}
For the $L$-layer normalized GCN classifier on a symmetric $k$-NN graph with degrees in $[k,w{-}1]$:
\begin{enumerate}
\item[(i)] if $\|\delta\mathbf{X}\|_F\le\epsilon$, the change in output logits is at most $\epsilon\prod_{l=1}^{L}\|\mathbf{W}^{(l)}\|_{\sigma}$;
\item[(ii)] modifying $m$ edges changes the normalized adjacency operator by at most $2m/k$ in Frobenius norm, so the logit change is at most $(2m/k)\prod_{l=1}^{L}\|\mathbf{W}^{(l)}\|_{\sigma}\,\|\mathbf{X}\|_F$.
\end{enumerate}
\end{proposition}

\emph{Proof}
(i) Each normalized GCN layer is $\hat{\mathbf{D}}^{-1/2}\hat{\mathbf{A}}\hat{\mathbf{D}}^{-1/2}\mathbf{H}\mathbf{W}$; the symmetric normalized operator has spectral norm at most $1$, and ReLU is $1$-Lipschitz, so the per-layer Lipschitz constant is $\|\mathbf{W}^{(l)}\|_\sigma$ and they compose multiplicatively. (ii) By Lemma~\ref{lem:knn_properties} every degree is at least $k$, so each entry of $\hat{\mathbf{D}}^{-1/2}$ is at most $1/\sqrt{k}$; a single edge edit perturbs the operator by at most $2/k$ in Frobenius norm, and $m$ edits by at most $2m/k$. Propagating through the layers gives the bound.

\noindent Part (ii) shows graceful degradation: a few mis-assigned $k$-NN edges arising from feature noise cannot flip predictions abruptly, with the tolerance scaling as $1/k$.

\subsubsection{Data Augmentation Strategy for Class Imbalance}
Our strategy leverages the trained generator to create synthetic DDoS samples, balancing class distribution and improving the classifier's ability to recognize attack patterns. The augmented dataset combines original and synthetic samples as
\begin{equation}\label{eq:daug}
\mathcal{D}_{\text{aug}} = \mathcal{D}_{\text{real}} \cup \mathcal{D}_{\text{syn}},
\end{equation}
where $\mathcal{D}_{\text{real}} = \{(\mathcal{G}_i, y_i)\}_{i=1}^{N_{\text{real}}}$ contains original graphs and $\mathcal{D}_{\text{syn}} = \{(\mathcal{G}_{\text{syn}}^{(i)}, 1)\}_{i=1}^{N_{\text{syn}}}$ contains synthetic DDoS graphs. The number of synthetic samples is determined by:
\begin{equation}\label{eq:nsyn}
N_{\text{syn}} = \min(\max(0, N_0 - N_1), N_{\text{max}}),
\end{equation}
where $N_0$ and $N_1$ are benign and DDoS samples, and $N_{\text{max}}$ prevents excessive generation that may cause noise or overfitting. Each synthetic graph $\mathcal{G}_{\text{syn}}^{(i)}$ is built from a template structure of real DDoS samples, populated with generator-produced features:
\begin{equation}\label{eq:gsyn}
\mathcal{G}_{\text{syn}}^{(i)} = (\mathcal{V}_{\text{template}}, \mathcal{E}_{\text{template}}, G_\theta(\mathbf{Z}^{(i)})).
\end{equation}
Together, Eqs.~\eqref{eq:daug}--\eqref{eq:gsyn} define a principled augmentation pipeline that respects graph topology. The following corollary shows that this strategy provably reduces class imbalance, achieving perfect balance when the generation budget is sufficiently large.

\begin{corollary}[Class Balance Guarantee]
\label{cor:class_balance}
Let $N_0$ and $N_1$ denote the number of benign and DDoS samples, respectively, with $N_0 > N_1$ (class imbalance). After augmentation with $N_{\text{syn}} = \min(N_0 - N_1, N_{\text{max}})$ synthetic DDoS samples, the class imbalance ratio satisfies
$
\rho = \frac{N_1 + N_{\text{syn}}}{N_0} = \begin{cases} 1, & \text{if } N_0 - N_1 \leq N_{\text{max}}, \\ \frac{N_1 + N_{\text{max}}}{N_0}, & \text{otherwise}. \end{cases}
$
In particular, when $N_{\text{max}} \geq N_0 - N_1$, the augmented dataset achieves perfect class balance ($\rho = 1$).
\end{corollary}

\noindent By Theorem~\ref{thm:gan_convergence}, the synthetic samples generated at convergence are drawn from $p_G \approx p_{\text{data}}$, ensuring that the balanced dataset preserves the statistical properties of real DDoS traffic rather than introducing distributional artifacts.

Corollary~\ref{cor:class_balance} establishes that augmentation achieves class balance; the next result quantifies what that balance is \emph{worth} for generalization by connecting GAN training quality to the downstream classifier through the domain-adaptation bound of Ben-David et al.

\begin{theorem}[Generalization Gap Controlled by Discriminator Error]
\label{thm:gen}
Let $C_\psi$ be the GCN classifier trained on $\mathcal{D}_{\mathrm{aug}}= \mathcal{D}_{\mathrm{real}}\cup\mathcal{D}_{\mathrm{syn}}$, and let $R_{\mathrm{real}}(C)$ be its risk had it trained on an equally sized balanced set of real DDoS graphs. With $\varepsilon_D$ the optimal discriminator's real-vs-synthetic error rate at convergence and $\lambda^{\star}$ the joint risk of the best classifier over both distributions,
\[
R_{\mathrm{real}}(C) \;\le\; R_{\mathrm{aug}}(C) \;+\; 2\,(1-2\varepsilon_D) \;+\; \lambda^{\star}.
\]
\end{theorem}

\emph{Proof}
Apply the Ben-David {\itshape et al.}\ target-risk bound \cite{bendavid2010theory} with the augmented distribution as source and the real balanced distribution as target. The $\mathcal{H}\Delta\mathcal{H}$ divergence between the two is estimated by the $\mathcal{A}$-distance $d_{\mathcal{A}} = 2(1-2\varepsilon_D)$, where $\varepsilon_D$ is the error of the optimal real-vs-synthetic discriminator. Substituting yields the bound.

\noindent As training approaches the equilibrium of Theorem~\ref{thm:gan_convergence} ($D^{\star}\!\to\!\tfrac12$, $\varepsilon_D\!\to\!\tfrac12$), the divergence term $2(1-2\varepsilon_D)\!\to\!0$ and the augmented-data classifier generalizes as well as one trained on real balanced data, up to the irreducible $\lambda^{\star}$. This makes the empirical fake detection rate of 45.3\% reported in the generator-architecture ablation (Section~\ref{results}) a direct, measurable proxy for the generalization gap.

\begin{corollary}[Real-Sample Efficiency]
\label{cor:samples}
Under the coordinated regime of Definition~\ref{def:coord}, each real DDoS window carries up to $\bar m$ correlated same-class labels rather than a single independent one. The effective sample size per window is inflated accordingly, so to reach a target risk, the number of \emph{real} DDoS samples required by graph augmentation is smaller than that required by vector augmentation by the corresponding factor.
\end{corollary}

\subsubsection{Classification Training Objective}
The classifier is trained using the cross-entropy loss function, which is well-suited for probabilistic classification tasks and provides gradient information that encourages confident predictions on correctly classified samples while heavily penalizing confident misclassifications. The training objective minimizes the cross-entropy loss over the augmented dataset as
\begin{equation}
\mathcal{L}_C = -\frac{1}{|\mathcal{D}_{\text{aug}}|}\sum_{(\mathcal{G}, y) \in \mathcal{D}_{\text{aug}}} \sum_{c=0}^{1} y_c \log P(y=c|\mathcal{G}),
\end{equation}
where $y_c$ represents the one-hot encoding of the true label, such that $y_c = 1$ if $c$ is the correct class and $y_c = 0$ otherwise. This formulation ensures that confident, correct predictions yield small updates, while uncertain or incorrect predictions drive stronger corrective updates. Indeed, for the softmax output, the gradient of the cross-entropy loss with respect to the pre-activation logits $\mathbf{o} = \mathbf{W}_C \mathbf{h}_{C,\text{pool}} + \mathbf{b}_C$ reduces to the well-known compact form
\begin{equation}
\frac{\partial \mathcal{L}_C}{\partial \mathbf{o}} = \frac{1}{|\mathcal{D}_{\text{aug}}|}\sum_{(\mathcal{G}, y) \in \mathcal{D}_{\text{aug}}} \big(P(y \mid \mathcal{G}) - \mathbf{y}\big),
\end{equation}
where $\mathbf{y}$ is the one-hot label vector. The update is thus proportional to the prediction error $P(y \mid \mathcal{G}) - \mathbf{y}$: it vanishes when the predicted distribution matches the true label and grows toward its maximum for confident misclassifications, supplying strong corrective gradients precisely where the classifier errs.

\section{Experimental Results}\label{results}
In this section, we present the experimental settings, quantitative results, and comprehensive ablation studies of the proposed GraphGAN framework across four benchmark datasets.

\subsection{Experimental Settings and Datasets}
We evaluate GraphGAN on four widely used network intrusion detection benchmarks to demonstrate its generalizability:

\noindent\textbf{CIC-IDS-2017} contains approximately 2.8 million flow records collected over seven days, encompassing 14 attack types (e.g., DDoS, PortScan, Web Attack, Infiltration, Botnet) and normal flows. Each record includes 80 network flow features.

\noindent\textbf{CIC-IDS-2018} extends the 2017 dataset with approximately 16--17 million records collected over 10 days, covering seven attack scenarios, including Brute Force, DoS, DDoS, Web Attack, Infiltration, Botnet, and Heartbleed, and includes 80 extracted flow-level features.

\noindent\textbf{UNSW-NB15} comprises approximately 2.5 million records generated in a hybrid real and synthetic environment, containing nine attack categories (Fuzzers, Analysis, Backdoors, DoS, Exploits, Generic, Reconnaissance, Shellcode, Worms) with 49 flow-based and content-based features.

\noindent\textbf{ToN-IoT} is an IoT/IIoT-focused dataset with approximately 461,000 records collected from heterogeneous IoT services (weather, fridge, garage door, GPS, modbus, light, motion, thermostat), encompassing nine attack types including DDoS, DoS, ransomware, backdoor, injection, XSS, password cracking, scanning, and man-in-the-middle, with 44 features.

For all datasets, the GraphGAN framework employs consistent hyperparameters unless otherwise stated. The generator synthesizes node features from Gaussian noise ($z=128$), while the discriminator operates on graph-structured data using three GCN layers. Graph construction uses $k$-NN with $k=5$, window size $w=30$, and step size $s=10$. The GCN-based classifier uses three graph convolutional layers, followed by global mean pooling and fully connected layers, and is trained with the Adam optimizer (learning rate $0.001$), dropout regularization ($p=0.3$), and cross-entropy loss. All models were implemented in PyTorch with the PyTorch Geometric backend and trained on a GPU with a batch size of $32$ for up to $50$ epochs. To prevent temporal leakage arising from overlapping sliding windows ($s<w$), we order each dataset chronologically and partition the raw flow stream into 60\% training, 20\% validation, and 20\% testing segments \emph{before} graph construction. Sliding windows are then generated independently within each segment, so that no window crosses a split boundary and no raw flow is shared across splits; a guard band of $w$ flows is discarded at each boundary to eliminate residual overlap. This chronological protocol (training on earlier traffic and testing on later traffic) reflects realistic deployment and avoids the inflated accuracy that random window-level splitting can produce. Performance is evaluated primarily using classification accuracy, while precision, recall, and F1-score are reported in subsequent analyses. We stress that all baselines are re-implemented under an identical, leakage-controlled protocol: the flow stream is ordered chronologically and partitioned \emph{before} windowing, with a guard band discarded at each boundary. This is a strictly harder evaluation than the random, window-level splits under which several published graph baselines report near-saturated accuracy; accordingly, our reproduced figures for these baselines are lower than their originally reported values, and the gap reflects protocol rigor rather than implementation disadvantage.

\begin{table*}[!b]
\centering
\caption{Comparative analysis of GraphGAN against state-of-the-art methods across four benchmark datasets (test accuracy \%).}
\label{tab:comparison}
\resizebox{0.7\textwidth}{!}{%
\begin{tabular}{lcccc}
\toprule
\textbf{Method} & \textbf{CIC-IDS-2017} & \textbf{CIC-IDS-2018} & \textbf{UNSW-NB15} & \textbf{ToN-IoT} \\
\midrule
Causal Deep Learning \cite{zeng2022improving} & 82.35 & 81.23 & 79.56 & 80.89 \\
CNN \cite{najar2024cyber} & 90.56 & 89.78 & 87.34 & 88.56 \\
GRU-BiLSTM \cite{al2025deep} & 90.35 & 89.45 & 87.12 & 88.34 \\
SMOTE Oversampling \cite{chawla2002smote} & 84.72 & 84.08 & 82.19 & 82.87 \\
Vanilla GAN-based Classifier & 86.43 & 85.67 & 83.45 & 84.56 \\
WGAN \cite{arafah2025anomaly} & 88.46 & 87.92 & 85.76 & 86.48 \\
VAE-GAN \cite{tian2024vae} & 89.23 & 88.64 & 86.31 & 87.12 \\
DDP-DAR (Diffusion) \cite{cai2025ddp} & 91.48 & 90.76 & 88.95 & 89.62 \\
E-GraphSAGE \cite{lo2022graphsage} & 92.03 & 91.35 & 89.41 & 90.18 \\
BS-GAT \cite{wang2025bs} & 93.18 & 92.47 & 90.86 & 91.73 \\
Ensemble GNN \cite{bakar2024ftg} & 92.55 & 91.89 & 89.78 & 90.67 \\
\midrule
\textbf{GraphGAN (Proposed)} & \textbf{95.31} & \textbf{94.87} & \textbf{93.42} & \textbf{94.56} \\
\bottomrule
\end{tabular}}
\end{table*}

\subsection{Quantitative Results and Analysis}
To ensure a fair and comprehensive evaluation, we compare GraphGAN against representative state-of-the-art methods spanning different architectural paradigms in network intrusion detection.
Causal Deep Learning \cite{zeng2022improving} represents feature-level causal modeling approaches designed to enhance interpretability and robustness. CNN \cite{najar2024cyber} captures spatial correlations in features through convolutional filters and serves as a strong deep learning baseline for tabular traffic data. GRU-BiLSTM \cite{al2025deep} models sequential dependencies in traffic flows using recurrent architectures. The Vanilla GAN-based classifier integrates adversarial training without a graph structure to evaluate whether generative augmentation alone is sufficient. Ensemble GNN \cite{bakar2024ftg} represents advanced graph-based detection approaches without adversarial augmentation. To stress-test our claims against the strongest generative and graph paradigms, we additionally include: SMOTE oversampling \cite{chawla2002smote} as a classical interpolation-based augmenter; a Wasserstein GAN \cite{arafah2025anomaly}, a VAE-GAN \cite{tian2024vae}, and the diffusion-based DDP-DAR \cite{cai2025ddp} as advanced vectorized generative augmenters; and E-GraphSAGE \cite{lo2022graphsage} and BS-GAT \cite{wang2025bs} as state-of-the-art graph neural baselines.
All baselines were implemented using identical preprocessing and data-splitting protocols to ensure a fair comparison.

We evaluated all baseline methods and GraphGAN across all four benchmark datasets using 60\% of the training data, and the results are shown in Table~\ref{tab:comparison}. On CIC-IDS-2017, GraphGAN achieves 95.31\% accuracy, substantially outperforming all competitors. The pattern is consistent across datasets: GraphGAN attains 94.87\% on CIC-IDS-2018, 93.42\% on UNSW-NB15, and 94.56\% on ToN-IoT. Traditional methods show limited effectiveness across all benchmarks: Causal Deep Learning (79.56\%--82.35\%), CNN (87.34\%--90.56\%), GRU-BiLSTM (87.12\%--90.35\%), and standalone GAN (83.45\%--86.43\%). Among the baselines, Ensemble GNN is the strongest, achieving 89.78\%--92.55\% accuracy but remaining 2.76--3.64 percentage points below GraphGAN across all datasets. GraphGAN's superiority stems from three synergistic innovations: temporal graph construction via $k$-NN topology that captures coordinated attack patterns, adversarial training that generates realistic minority samples while preserving topological relationships, and GCN-based classification that leverages structural patterns for robust boundaries. The consistent improvements across diverse datasets confirm the generalizability of our approach.

\subsection{Necessity of Data Augmentation Using GraphGAN}
While GraphGAN demonstrates strong overall accuracy, we further investigate whether its advantage stems specifically from adversarial augmentation in the presence of class imbalance. To validate the necessity of adversarial augmentation, we conducted ablation studies on CIC-IDS-2017 across varying class imbalance ratios (1:2, 1:5, 1:10, and 1:20), comparing GraphGAN with full augmentation against baselines without augmentation and with random feature generation. As shown in Table~\ref{tab:ablation}, GraphGAN consistently outperforms both variants across all metrics. At moderate imbalance (1:2), GraphGAN achieves 95.31\% accuracy, a 1.55 point improvement over the non-augmented baseline (93.76\%). The performance gap widens dramatically as imbalance intensifies: at 1:10 ratio, GraphGAN maintains 91.82\% accuracy while the baseline degrades to 78.91\% (12.91 point gap), and at extreme imbalance (1:20), GraphGAN achieves 89.47\% compared to 67.23\% without augmentation (22.24 point improvement). Random augmentation provides only marginal benefits (2.91 points at 1:20) and fails to capture distributional and topological properties. These results confirm that GraphGAN's adversarial training generates topology-aware synthetic samples that preserve both statistical fidelity and structural relationships.

\begin{table}[!t]
\centering
\caption{Quantitative Result of Class Imbalance Ratio for \emph{GraphGAN} on CIC-IDS-2017}
\label{tab:ablation}
\resizebox{\linewidth}{!}{%
\begin{tabular}{llcccc}
\toprule
\textbf{Ratio} & \textbf{Model Variant} & \textbf{Accuracy} & \textbf{Precision} & \textbf{Recall} & \textbf{F1 Score} \\
\midrule
\multirow{3}{*}{\textbf{1:2}}
& GraphGAN (Full) & \textbf{95.31} & \textbf{95.58} & \textbf{95.31} & \textbf{95.07} \\
& w/o Augmentation & 93.76 & 94.12 & 93.76 & 93.52 \\
& + Random & 93.89 & 94.28 & 93.89 & 93.64 \\
\midrule
\multirow{3}{*}{\textbf{1:5}}
& GraphGAN (Full) & \textbf{93.58} & \textbf{93.92} & \textbf{93.58} & \textbf{93.41} \\
& w/o Augmentation & 87.34 & 88.16 & 87.34 & 86.92 \\
& + Random & 88.12 & 88.79 & 88.12 & 87.68 \\
\midrule
\multirow{3}{*}{\textbf{1:10}}
& GraphGAN (Full) & \textbf{91.82} & \textbf{92.26} & \textbf{91.82} & \textbf{91.54} \\
& w/o Augmentation & 78.91 & 80.47 & 78.91 & 78.23 \\
& + Random & 80.67 & 81.92 & 80.67 & 79.94 \\
\midrule
\multirow{3}{*}{\textbf{1:20}}
& GraphGAN (Full) & \textbf{89.47} & \textbf{90.15} & \textbf{89.47} & \textbf{89.08} \\
& w/o Augmentation & 67.23 & 71.34 & 67.23 & 65.87 \\
& + Random & 70.14 & 73.28 & 70.14 & 68.79 \\
\bottomrule
\end{tabular}}
\end{table}

\begin{table}[!b]
\caption{Impact of (a) Graph Construction Method and (b) Pooling Strategy on Accuracy (\%)}
\label{tab:construction_pooling}
\centering
\begin{minipage}[t]{0.48\linewidth}
\centering
\footnotesize
\textbf{(a) Graph Construction}\\[3pt]
\resizebox{\linewidth}{!}{%
\begin{tabular}{@{}lcccc@{}}
\toprule
\textbf{Method} & \textbf{CIC17} & \textbf{CIC18} & \textbf{NB15} & \textbf{ToN} \\
\midrule
$k$-NN & \textbf{95.31} & \textbf{94.87} & \textbf{93.42} & \textbf{94.56} \\
Fully Conn. & 92.87 & 92.34 & 90.78 & 91.93 \\
Threshold & 91.43 & 90.91 & 89.56 & 90.67 \\
Random & 88.76 & 88.12 & 87.34 & 88.45 \\
\bottomrule
\end{tabular}}
\end{minipage}
\hfil
\begin{minipage}[t]{0.48\linewidth}
\centering
\footnotesize
\textbf{(b) Pooling Strategy}\\[3pt]
\resizebox{\linewidth}{!}{%
\begin{tabular}{@{}lcccc@{}}
\toprule
\textbf{Method} & \textbf{CIC17} & \textbf{CIC18} & \textbf{NB15} & \textbf{ToN} \\
\midrule
Mean & \textbf{95.31} & \textbf{94.87} & \textbf{93.42} & \textbf{94.56} \\
Max & 93.78 & 93.23 & 91.89 & 92.78 \\
Sum & 92.91 & 92.34 & 91.12 & 92.23 \\
Attention & 94.67 & 94.12 & 92.78 & 93.89 \\
\bottomrule
\end{tabular}}
\end{minipage}
\end{table}

\subsection{Impact of Different Training Ratio}
To evaluate the robustness and data efficiency of our proposed GraphGAN framework, we conducted experiments with training data ratios ranging from 10\% to 90\%, while maintaining a consistent test set composition and preserving the original class distribution. As illustrated in Fig.~\ref{fig:training_size_graphgan}, GraphGAN consistently achieves the highest accuracy across all training ratios, with particularly pronounced advantages under low-data regimes, where traditional deep learning approaches experience significant performance degradation. The superior data efficiency stems from a graph construction mechanism that captures stable flow relationships regardless of the training set size, an adversarial augmentation that generates high-quality synthetic samples while preserving topological relationships, and a GCN-based architecture that leverages parameter sharing to improve generalization under data-limited conditions.

\begin{figure}[!t]
    \centering
    \begin{minipage}[b]{0.48\linewidth}
        \centering
        \includegraphics[width=\linewidth]{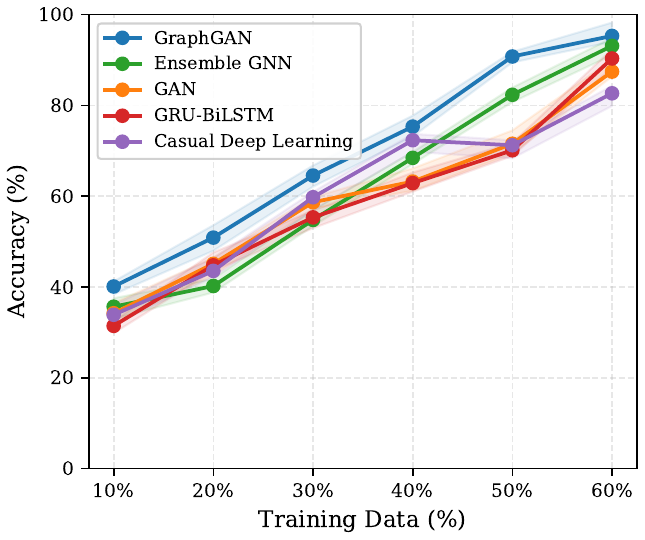}\\[2pt]
        {\small (a) Accuracy}
    \end{minipage}
    \hfil
    \begin{minipage}[b]{0.48\linewidth}
        \centering
        \includegraphics[width=\linewidth]{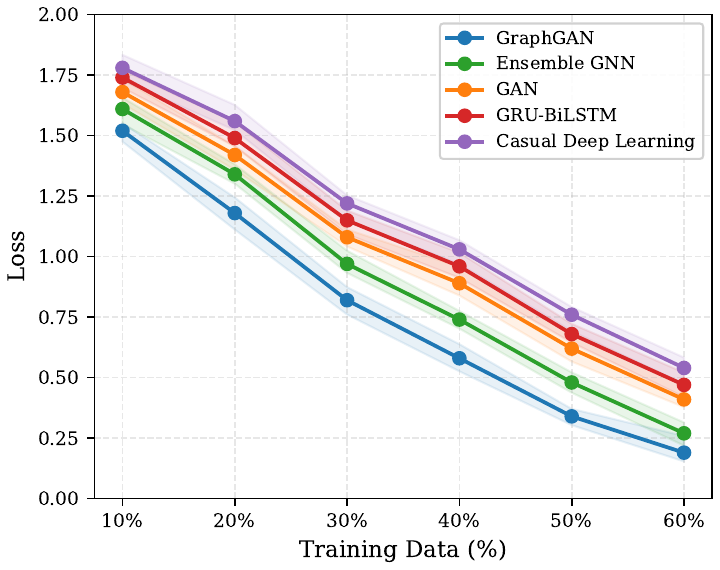}\\[2pt]
        {\small (b) Loss}
    \end{minipage}
    \caption{Impact of training data ratio on (a) accuracy and (b) loss of data-driven methods, including GraphGAN.}
    \label{fig:training_size_graphgan}
\end{figure}

\subsection{Impact of Graph Construction Method}
To validate our $k$-NN graph construction approach, we compared it against alternative graph topologies: (1) fully connected graphs, where all nodes within a temporal window are connected; (2) threshold-based graphs, where edges are formed only if similarity exceeds a fixed threshold; and (3) random graphs with the same average degree. As shown in Table~\ref{tab:construction_pooling}(a), the $k$-NN approach consistently achieves the highest accuracy across all four datasets, reaching 95.31\% on CIC-IDS-2017, 94.87\% on CIC-IDS-2018, 93.42\% on UNSW-NB15, and 94.56\% on ToN-IoT. Fully connected graphs suffer from over-connectivity, which introduces noise (90.78\%--92.87\%); threshold-based methods fragment the graph into disconnected components (89.56\%--91.43\%); and random graphs lack meaningful structural information (87.34\%--88.76\%). The $k$-NN approach optimally balances the preservation of local structure and computational efficiency across diverse network environments.

\subsection{Effect of Window Size and Step Size}
We investigated the impact of temporal window size $w \in \{10, 20, 30, 40, 50\}$ and step size $s \in \{5, 10, 15, 20\}$ on detection performance of our proposed method (GraphGAN) across all four datasets. As illustrated in Fig.~\ref{fig:window_step}, the configuration $w=30$ and $s=10$ yields optimal results across all benchmarks: 95.31\% on CIC-IDS-2017, 94.87\% on CIC-IDS-2018, 93.42\% on UNSW-NB15, and 94.56\% on ToN-IoT. Smaller windows ($w=10$) fail to capture extended attack patterns (89.56\%--91.23\%), whereas larger windows ($w=50$) introduce noise from unrelated flows (91.45\%--92.67\%). Regarding the step size, $s=10$ provides the best trade-off: excessively large steps ($s=20$) may miss transitional attack patterns, whereas very small steps ($s=5$) produce redundant, overlapping windows without significant performance improvement. These trends remain consistent across all datasets.

\begin{figure}[!t]
    \centering
    \includegraphics[width=0.9\linewidth]{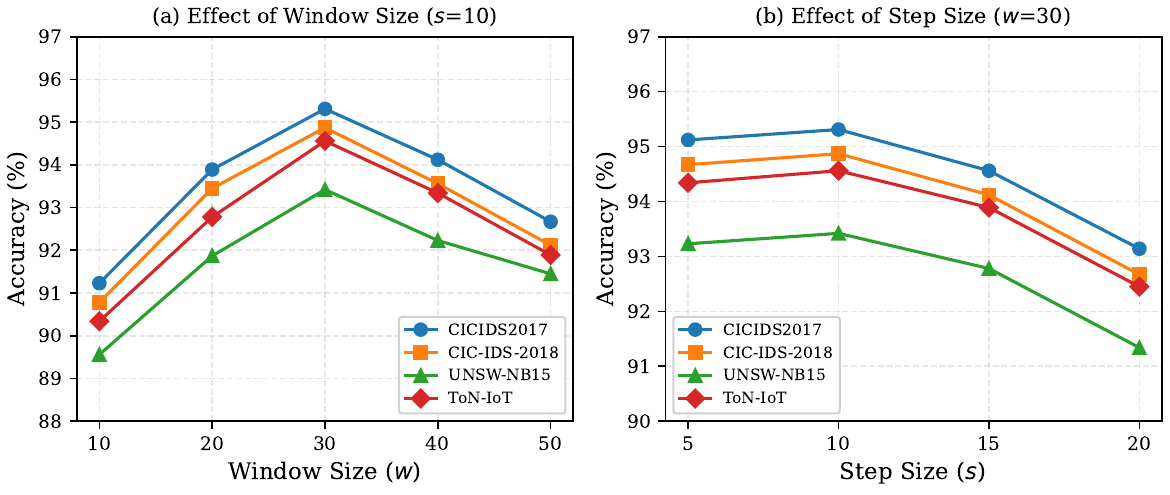}
    \caption{Effect of (a) window size $w$ and (b) step size $s$ on detection accuracy of our proposed GraphGAN.}
    \label{fig:window_step}
\end{figure}

\subsection{Importance of the Number of Nearest Neighbors ($k$)}
The choice of $k$ in $k$-NN graph construction critically affects the quality of the resulting topology. We evaluated $k \in \{3, 5, 7, 10, 15\}$ across all four datasets, with results presented in Fig.~\ref{fig:hyperparams}(a). Across all benchmarks, $k=5$ consistently provides the optimal balance between capturing local neighborhood structure and maintaining graph connectivity, achieving 95.31\% on CIC-IDS-2017, 94.87\% on CIC-IDS-2018, 93.42\% on UNSW-NB15, and 94.56\% on ToN-IoT. Lower values ($k=3$) produce overly sparse graphs with disconnected components (90.67\%--92.45\%), thereby limiting effective message passing. Higher values ($k=15$) introduce noisy long-range connections that dilute meaningful local patterns (91.78\%--93.78\%). The consistent optimality of $k=5$ across diverse datasets confirms that this connectivity level effectively captures feature-based similarities without degrading graph topology.

\subsection{Analysis of GCN Layer Depth}
We examined the impact of GCN depth by varying the number of graph convolutional layers in both the discriminator and the classifier from 1 to 5. As shown in Fig.~\ref{fig:hyperparams}(b), the three-layer architecture achieves the highest accuracy across all datasets: 95.31\% on CIC-IDS-2017, 94.87\% on CIC-IDS-2018, 93.42\% on UNSW-NB15, and 94.56\% on ToN-IoT. Shallow networks (1--2 layers) capture only immediate neighborhood information, thereby missing multi-hop structural patterns essential for distributed attack detection. Conversely, deeper networks (4--5 layers) suffer from over-smoothing, in which node representations become increasingly indistinguishable, as well as from gradient vanishing effects that impede effective training. The three-layer configuration provides an optimal balance between receptive field expansion and representation discrimination, a finding that remains consistent across all four benchmarks.

\subsection{Study of Generator Architecture Complexity}
To assess the necessity of the proposed multi-layer generator architecture, we compared variants with different depths: 2-layer (linear output), 3-layer (a 4-layer architecture without one hidden layer), 4-layer (the proposed architecture), and 5-layer. The 4-layer generator produces the highest-quality synthetic samples, as measured by discriminator confusion (45.3\% fake detection rate) and downstream classifier performance (95.31\% accuracy). Simpler 2-layer generators produce easily distinguishable synthetic samples (78.2\% fake detection rate, 89.45\% classifier accuracy), while 5-layer generators exhibit training instability and only marginal improvements (47.1\% fake detection rate, 94.89\% accuracy), indicating diminishing returns beyond four layers.

\begin{figure}[!t]
\centering
\begin{minipage}[b]{0.48\linewidth}
    \centering
    \includegraphics[width=\linewidth]{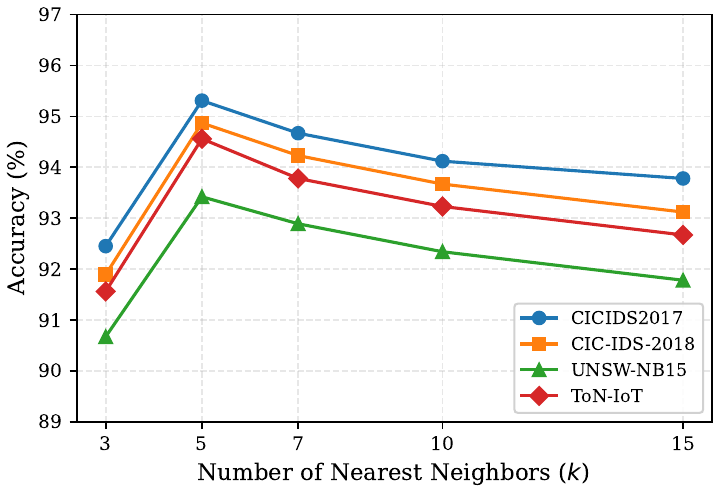}\\[2pt]
    {\small (a)}
\end{minipage}
\hfil
\begin{minipage}[b]{0.48\linewidth}
    \centering
    \includegraphics[width=\linewidth]{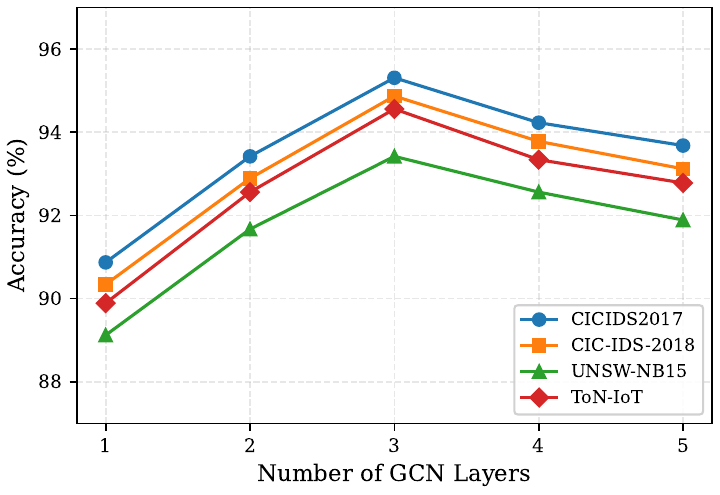}\\[2pt]
    {\small (b)}
\end{minipage}
\caption{Effect of hyperparameters on detection accuracy of our proposed GraphGAN: (a) number of nearest neighbors $k$, and (b) GCN layer depth.}
\label{fig:hyperparams}
\end{figure}

\subsection{Aggregation Function in Graph Pooling}
We compared different graph-level pooling strategies across all four datasets: global mean pooling (proposed), global max pooling, global sum pooling, and attention-based pooling. As shown in Table~\ref{tab:construction_pooling}(b), mean pooling consistently achieves the highest accuracy across all benchmarks: 95.31\% on CIC-IDS-2017, 94.87\% on CIC-IDS-2018, 93.42\% on UNSW-NB15, and 94.56\% on ToN-IoT. Max pooling (91.89\%--93.78\%) focuses on extreme features but loses information about overall flow patterns. Sum pooling (91.12\%--92.91\%) is sensitive to variations in graph size and introduces scale inconsistencies. Attention-based pooling (92.78\%--94.67\%) shows promise but introduces additional parameters and training complexity without sufficient performance gains to justify the computational overhead.

\section{Concluding Remarks}\label{conclusion}
We introduced GraphGAN, an adversarial graph neural network framework that models network traffic as temporal graphs and employs generative augmentation to address class imbalance and data scarcity in DDoS detection. Through temporal graph construction, adversarial augmentation, and GCN-based classification, GraphGAN achieves superior performance compared to existing CNN-LSTM variants across four diverse benchmarks, including CIC-IDS-2017, CIC-IDS-2018, UNSW-NB15, and ToN-IoT, with consistent improvements in accuracy, precision, and recall. Comprehensive ablation studies confirm the robustness of key design choices, including $k$-NN graph construction, three-layer GCN depth, a noise dimension of 128, and mean pooling, across all evaluated datasets. Key limitations include reliance on feature-similarity assumptions in graph construction and training instabilities inherent to adversarial architectures. A detailed system-architecture instantiation and its runtime/mitigation-cost evaluation are left as future work. Future research directions encompass multi-vector attack classification, hierarchical graph structures, federated learning frameworks, explainable threat attribution mechanisms, and extension to broader network security domains, including encrypted traffic analysis and IoT anomaly detection.

\section*{Data Availability}
The datasets analyzed in this study are publicly available benchmark datasets: CIC-IDS-2017, CIC-IDS-2018, UNSW-NB15, and ToN-IoT. Code to reproduce the experiments is available from the authors on reasonable request.

\bibliographystyle{IEEEtran}
\bibliography{references}

@ARTICLE{Ansari_Zerotrust,
  author    = {Hichem Sedjelmaci and Nirwan Ansari},
  title     = {Zero Trust Architecture Empowered Attack Detection Framework to Secure 6G Edge Computing},
  journal   = {IEEE Network},
  year      = {2024},
  volume    = {38},
  number    = {1},
  pages     = {196--202},
  month     = jan,
  publisher = {IEEE Press},
  issn      = {0890-8044},
  doi       = {10.1109/MNET.131.2200513}
}

@article{kullback1951information,
  author  = {Solomon Kullback and Richard A. Leibler},
  title   = {On Information and Sufficiency},
  journal = {The Annals of Mathematical Statistics},
  year    = {1951},
  volume  = {22},
  number  = {1},
  pages   = {79--86},
  doi     = {10.1214/aoms/1177729694},
  publisher = {Institute of Mathematical Statistics}
}

@article{bendavid2010theory,
  author    = {Shai Ben-David and John Blitzer and Koby Crammer and Alex Kulesza and Fernando Pereira and Jennifer Wortman Vaughan},
  title     = {A Theory of Learning from Different Domains},
  journal   = {Machine Learning},
  year      = {2010},
  volume    = {79},
  number    = {1--2},
  pages     = {151--175},
  doi       = {10.1007/s10994-009-5152-4},
  publisher = {Springer}
}

@ARTICLE{Nishanth_DDoS,
  author={Nishanth, N. and Mujeeb, A.},
  journal={IEEE Transactions on Cognitive Communications and Networking}, 
  title={Modeling and Detection of Flooding-Based Denial of Service Attacks in Wireless Ad Hoc Networks Using Uncertain Reasoning}, 
  year={2021},
  volume={7},
  number={3},
  pages={893-904},
  doi={10.1109/TCCN.2021.3055503}}

@ARTICLE{Hossain_SL,
  author={Hossain, Mohammad Arif and Sadat, Nazmus and Ansari, Nirwan and Amsaad, Fathi},
  journal={IEEE Transactions on Cognitive Communications and Networking}, 
  title={Split Learning Over NOMA-Enabled HetNets for Scalable IoT-Based Precision Agriculture}, 
  year={2026},
  volume={12},
  number={},
  pages={7145-7156},
  doi={10.1109/TCCN.2026.3683140}}

@ARTICLE{Wang_DDoS,
  author={Wang, Yipeng and Zhang, Xintong and Lai, Yingxu and Zhao, Zijian and Deng, Yongjian},
  journal={IEEE Transactions on Cognitive Communications and Networking}, 
  title={Hifoots: A Highly Efficient DDoS Attack Detection Scheme Deployed in Smart IoT Homes}, 
  year={2025},
  volume={11},
  number={1},
  pages={519-533},
  doi={10.1109/TCCN.2024.3424888}}

@ARTICLE{Shahraki_Net_Traffic,
  author={Shahraki, Amin and Abbasi, Mahmoud and Taherkordi, Amir and Jurcut, Anca Delia},
  journal={IEEE Transactions on Cognitive Communications and Networking}, 
  title={Active Learning for Network Traffic Classification: A Technical Study}, 
  year={2022},
  volume={8},
  number={1},
  pages={422-439},
  doi={10.1109/TCCN.2021.3119062}}

@ARTICLE{Liu_RL,
  author={Liu, Weiqi and Hossain, Mohammad Arif and Ansari, Nirwan and Kiani, Abbas and Saboorian, Tony},
  journal={IEEE Transactions on Network and Service Management}, 
  title={Reinforcement Learning-Based Network Slicing Scheme for Optimized UE-QoS in Future Networks}, 
  year={2024},
  volume={21},
  number={3},
  pages={3454-3464},
  doi={10.1109/TNSM.2024.3368294}}

@article{hossain2023decentralized,
  title={A decentralized collaborative learning approach in {5G+} core networks},
  author={Hossain, Mohammad Arif and Hossain, Abdullah Ridwan and Liu, Weiqi and Ansari, Nirwan and Kiani, Abbas and Saboorian, Tony},
  journal={IEEE Network},
  volume={38},
  number={1},
  pages={288--295},
  year={2023},
  publisher={IEEE}
}

@article{hossain2023hybrid,
  title={Hybrid multiple access for network slicing aware mobile edge computing},
  author={Hossain, Mohammad Arif and Ansari, Nirwan},
  journal={IEEE Transactions on Cloud Computing},
  volume={11},
  number={3},
  pages={2910--2921},
  year={2023},
  publisher={IEEE}
}

@ARTICLE{liu_DT,
  author={Liu, Weiqi and Arif Hossain, Mohammad and Ansari, Nirwan},
  journal={IEEE Transactions on Cognitive Communications and Networking}, 
  title={Mobile-Edge Computing for Multi-Services Digital Twin-Enabled IoT Heterogeneous Networks}, 
  year={2025},
  volume={11},
  number={3},
  pages={1845-1853},
  doi={10.1109/TCCN.2024.3490779}}

@ARTICLE{hossain_metaRL,
  author={Hossain, Mohammad Arif and Liu, Weiqi and Ansari, Nirwan},
  journal={IEEE Internet of Things Journal}, 
  title={Computation-Efficient Offloading and Power Control for MEC in IoT Networks by Meta-Reinforcement Learning}, 
  year={2024},
  volume={11},
  number={9},
  pages={16722-16730},
  doi={10.1109/JIoT.2024.3355023}}

@ARTICLE{Sayed_DDoS,
  author={Sayed, Mahmoud Said El and Le-Khac, Nhien-An and Azer, Marianne A. and Jurcut, Anca D.},
  journal={IEEE Transactions on Cognitive Communications and Networking}, 
  title={A Flow-Based Anomaly Detection Approach With Feature Selection Method Against DDoS Attacks in SDNs}, 
  year={2022},
  volume={8},
  number={4},
  pages={1862-1880},
  doi={10.1109/TCCN.2022.3186331}}

@ARTICLE{Luo_IDS,
  author={Luo, Yiqing and He, Mingshu and Wang, Xiaojuan},
  journal={IEEE Transactions on Cognitive Communications and Networking}, 
  title={ERFS: Efficient Feature Graph Representation for Intrusion Detection Based on Flow Semantic Association}, 
  year={2025},
  volume={11},
  number={6},
  pages={3711-3728},
  doi={10.1109/TCCN.2025.3543357}}

@article{mustapha2023detecting,
  title={Detecting DDoS attacks using adversarial neural network},
  author={Mustapha, Ali and Khatoun, Rida and Zeadally, Sherali and Chbib, Fadlallah and Fadlallah, Ahmad and Fahs, Walid and El Attar, Ali},
  journal={Computers \& Security},
  volume={127},
  pages={103117},
  year={2023},
  publisher={Elsevier}
}

@article{shieh2022detection,
  title={Detection of adversarial DDoS attacks using generative adversarial networks with dual discriminators},
  author={Shieh, Chin-Shiuh and Nguyen, Thanh-Tuan and Lin, Wan-Wei and Huang, Yong-Lin and Horng, Mong-Fong and Lee, Tsair-Fwu and Miu, Denis},
  journal={Symmetry},
  volume={14},
  number={1},
  pages={66},
  year={2022},
  publisher={MDPI}
}

@article{duan2022application,
  title={Application of a dynamic line graph neural network for intrusion detection with semisupervised learning},
  author={Duan, Guanghan and Lv, Hongwu and Wang, Huiqiang and Feng, Guangsheng},
  journal={IEEE Transactions on Information Forensics and Security},
  volume={18},
  pages={699--714},
  year={2022},
  publisher={IEEE}
}

@article{chawla2002smote,
  title={SMOTE: synthetic minority over-sampling technique},
  author={Chawla, Nitesh V and Bowyer, Kevin W and Hall, Lawrence O and Kegelmeyer, W Philip},
  journal={Journal of Artificial Intelligence Research},
  volume={16},
  pages={321--357},
  year={2002}
}

@article{al2025deep,
  title={A deep learning GRU-BiLSTM for DDoS attack detection},
  author={Al-Eryani, Amal M and Omara, Fatma A and Hossny, Eman},
  journal={SN Computer Science},
  volume={6},
  number={6},
  pages={605},
  year={2025},
  publisher={Springer}
}

@article{najar2024cyber,
  title={Cyber-secure SDN: A CNN-based approach for efficient detection and mitigation of DDoS attacks},
  author={Najar, Ashfaq Ahmad and Naik, S Manohar},
  journal={Computers \& Security},
  volume={139},
  pages={103716},
  year={2024},
  publisher={Elsevier}
}

@article{zeng2022improving,
  title={Improving the stability of intrusion detection with causal deep learning},
  author={Zeng, Zengri and Peng, Wei and Zeng, Detian},
  journal={IEEE Transactions on Network and Service Management},
  volume={19},
  number={4},
  pages={4750--4763},
  year={2022},
  publisher={IEEE}
}

@article{qing2024mitigating,
  title={Mitigating data imbalance to improve the generalizability in IoT DDoS detection tasks},
  author={Qing, Yi and Liu, Xiangyu and Du, Yanhui},
  journal={The Journal of Supercomputing},
  volume={80},
  number={7},
  pages={9935--9960},
  year={2024},
  publisher={Springer}
}

@article{cao2021detecting,
  title={Detecting and mitigating DDoS attacks in SDN using spatial-temporal graph convolutional network},
  author={Cao, Yongyi and Jiang, Hao and Deng, Yuchuan and Wu, Jing and Zhou, Pan and Luo, Wei},
  journal={IEEE Transactions on Dependable and Secure Computing},
  volume={19},
  number={6},
  pages={3855--3872},
  year={2021},
  publisher={IEEE}
}

@article{bakar2024ftg,
  title={FTG-Net-E: A hierarchical ensemble graph neural network for DDoS attack detection},
  author={Bakar, Rana Abu and De Marinis, Lorenzo and Cugini, Filippo and Paolucci, Francesco},
  journal={Computer Networks},
  volume={250},
  pages={110508},
  year={2024},
  publisher={Elsevier}
}

@article{jaafar2019review,
  title={Review of recent detection methods for HTTP DDoS attack},
  author={Jaafar, Ghafar A and Abdullah, Shahidan M and Ismail, Saifuladli},
  journal={Journal of Computer Networks and Communications},
  volume={2019},
  number={1},
  pages={1283472},
  year={2019},
  publisher={Wiley Online Library}
}

@inproceedings{li2022graphddos,
  title={Graphddos: Effective ddos attack detection using graph neural networks},
  author={Li, Yuzhen and Li, Renjie and Zhou, Zhou and Guo, Jiang and Yang, Wei and Du, Meijie and Liu, Qingyun},
  booktitle={2022 IEEE 25th International Conference on Computer Supported Cooperative Work in Design (CSCWD)},
  pages={1275--1280},
  year={2022},
  organization={IEEE}
}

@article{rajan2023detection,
  title={Detection and mitigation of {DDoS} attack in {SDN} environment using hybrid {CNN-LSTM}},
  author={Rajan, Dhanya M and Aravindhar, D John},
  journal={Migration Letters},
  volume={20},
  pages={407--419},
  year={2023}
}

@article{el2022flow,
  title={A flow-based anomaly detection approach with feature selection method against ddos attacks in sdns},
  author={El Sayed, Mahmoud Said and Le-Khac, Nhien-An and Azer, Marianne A and Jurcut, Anca D},
  journal={IEEE Transactions on Cognitive Communications and Networking},
  volume={8},
  number={4},
  pages={1862--1880},
  year={2022},
  publisher={IEEE}
}

@article{chouhan2023framework,
  title={A framework to detect DDoS attack in Ryu controller based software defined networks using feature extraction and classification},
  author={Chouhan, Ravindra Kumar and Atulkar, Mithilesh and Nagwani, Naresh Kumar},
  journal={Applied Intelligence},
  volume={53},
  number={4},
  pages={4268--4288},
  year={2023},
  publisher={Springer}
}

@inproceedings{antonakakis2017understanding,
  title={Understanding the mirai botnet},
  author={Antonakakis, Manos and April, Tim and Bailey, Michael and Bernhard, Matt and Bursztein, Elie and Cochran, Jaime and Durumeric, Zakir and Halderman, J Alex and Invernizzi, Luca and Kallitsis, Michalis and others},
  booktitle={26th USENIX security symposium (USENIX Security 17)},
  pages={1093--1110},
  year={2017}
}

@article{zargar2013survey,
  title={A survey of defense mechanisms against distributed denial of service (DDoS) flooding attacks},
  author={Zargar, Saman Taghavi and Joshi, James and Tipper, David},
  journal={IEEE communications surveys \& tutorials},
  volume={15},
  number={4},
  pages={2046--2069},
  year={2013},
  publisher={IEEE}
}

@article{zhao2023research,
  title={Research on data imbalance in intrusion detection using CGAN},
  author={Zhao, Guangyu and Liu, Peng and Sun, Ke and Yang, Yang and Lan, Tianyu and Yang, Han},
  journal={Plos one},
  volume={18},
  number={10},
  pages={e0291750},
  year={2023},
  publisher={Public Library of Science San Francisco, CA USA}
}

@article{zhong2024survey,
  title={A survey on graph neural networks for intrusion detection systems: Methods, trends and challenges},
  author={Zhong, Meihui and Lin, Mingwei and Zhang, Chao and Xu, Zeshui},
  journal={Computers \& Security},
  volume={141},
  pages={103821},
  year={2024},
  publisher={Elsevier}
}

@article{arafah2025anomaly,
  title={Anomaly-based network intrusion detection using denoising autoencoder and Wasserstein GAN synthetic attacks},
  author={Arafah, Mohammad and Phillips, Iain and Adnane, Asma and Hadi, Wael and Alauthman, Mohammad and Al-Banna, Abedal-Kareem},
  journal={Applied Soft Computing},
  volume={168},
  pages={112455},
  year={2025},
  publisher={Elsevier}
}

@article{tian2024vae,
  title={VAE-WACGAN: an improved data augmentation method based on VAEGAN for intrusion detection},
  author={Tian, Wuxin and Shen, Yanping and Guo, Na and Yuan, Jing and Yang, Yanqing},
  journal={Sensors},
  volume={24},
  number={18},
  pages={6035},
  year={2024},
  publisher={MDPI}
}

@article{cai2025ddp,
  title={DDP-DAR: Network intrusion detection based on denoising diffusion probabilistic model and dual-attention residual network},
  author={Cai, Saihua and Zhao, Yingwei and Lyu, Jiaao and Wang, Shengran and Hu, Yikai and Cheng, Mengya and Zhang, Guofeng},
  journal={Neural Networks},
  volume={184},
  pages={107064},
  year={2025},
  publisher={Elsevier}
}

@article{wang2025bs,
  title={BS-GAT: a network intrusion detection system based on graph neural network for edge computing},
  author={Wang, Yalu and Han, Zhijie and Du, Ying and Li, Jie and He, Xin},
  journal={Cybersecurity},
  volume={8},
  number={1},
  pages={27},
  year={2025},
  publisher={Springer}
}

@article{li2021fleam,
  title={FLEAM: A federated learning empowered architecture to mitigate DDoS in industrial IoT},
  author={Li, Jianhua and Lyu, Lingjuan and Liu, Ximeng and Zhang, Xuyun and Lyu, Xixiang},
  journal={IEEE Transactions on Industrial Informatics},
  volume={18},
  number={6},
  pages={4059--4068},
  year={2021},
  publisher={IEEE}
}

@article{he2024efficient,
  title={Efficient Based on Improved Random Forest Defense System Against Application-Layer DDoS Attacks},
  author={He, Junjiang and Fang, Wenbo and Lan, Xiaolong and Yang, Geying and Chen, Ziyu and Chen, Yang and Li, Tao and Chen, Jiangchuan},
  journal={International Journal of Intelligent Systems},
  volume={2024},
  number={1},
  pages={9044391},
  year={2024},
  publisher={Wiley Online Library}
}

@article{defferrard2016convolutional,
  title={Convolutional neural networks on graphs with fast localized spectral filtering},
  author={Defferrard, Micha{\"e}l and Bresson, Xavier and Vandergheynst, Pierre},
  journal={Advances in neural information processing systems},
  volume={29},
  year={2016}
}

@article{goodfellow2014generative,
  title={Generative adversarial nets},
  author={Goodfellow, Ian J and Pouget-Abadie, Jean and Mirza, Mehdi and Xu, Bing and Warde-Farley, David and Ozair, Sherjil and Courville, Aaron and Bengio, Yoshua},
  journal={Advances in neural information processing systems},
  volume={27},
  year={2014}
}

@article{xu2018powerful,
  title={How powerful are graph neural networks?},
  author={Xu, Keyulu and Hu, Weihua and Leskovec, Jure and Jegelka, Stefanie},
  journal={arXiv preprint arXiv:1810.00826},
  year={2018}
}

@article{yang2024improved,
  title={An improved intrusion detection method for {IIoT} using attention mechanisms, BiGRU, and Inception-CNN},
  author={Yang, Kai and Wang, JiaMing and Li, MinJing},
  journal={Scientific Reports},
  volume={14},
  number={1},
  pages={19339},
  year={2024},
  publisher={Nature Publishing Group UK London}
}

@inproceedings{lo2022graphsage,
  title={E-graphsage: A graph neural network based intrusion detection system for IoT},
  author={Lo, Wai Weng and Layeghy, Siamak and Sarhan, Mohanad and Gallagher, Marcus and Portmann, Marius},
  booktitle={NOMS 2022-2022 IEEE/IFIP network operations and management symposium},
  pages={1--9},
  year={2022},
  organization={IEEE}
}

\end{document}